\documentclass{article}

    \usepackage[preprint]{neurips_2020}

\usepackage[utf8]{inputenc} 
\usepackage[T1]{fontenc}    
\usepackage{hyperref}       
\usepackage{url}            
\usepackage{booktabs}       
\usepackage{amsfonts}       
\usepackage{nicefrac}       
\usepackage{microtype}      

\usepackage[english]{babel}
\setcitestyle{authoryear,open={(},close={)}}

\usepackage[utf8]{inputenc}
\usepackage{amsmath}
\usepackage{amssymb}
\usepackage{amsthm}
\usepackage{physics}
\usepackage{gensymb}
\usepackage{verbatim}
\usepackage{amsfonts}
\usepackage{mathrsfs}
\usepackage{tikz-cd}
\usepackage{graphicx}
\usepackage{microtype}
\usepackage{centernot}
\usepackage{mathtools}
\usepackage{algorithm}
\usepackage{algpseudocode}
\usepackage{bm}
\usepackage{caption}
\usepackage{enumitem}
\usepackage{cleveref}
\usepackage{subcaption}

\usepackage{pgfplots}
\pgfplotsset{compat=newest}
\usepackage{tikz}
\usetikzlibrary{shapes, positioning, quotes}

\pgfplotsset{height=8cm, width=15cm,compat=1.9}

\DeclareMathOperator*{\argmin}{arg\,min}

\newcommand{\Ex}{\mathbb{E}}
\newcommand{\R}{\mathbb{R}}

\newcommand{\bellman}{\mathcal{T}}
\newcommand{\statespace}{\mathcal{S}}
\newcommand{\actionspace}{\mathcal{A}}

\newcommand{\vpi}{V^\pi}
\newcommand{\qpi}{Q^\pi}
\newcommand{\qstar}{Q^*}

\newcommand{\functionspace}{\mathcal{F}}

\newcommand{\policyparams}{\theta}

\newcommand{\ipepi}{\pi_{IPE}}

\newtheorem{theorem}{Theorem}
\newtheorem{proposition}{Proposition}

\title{Inverse Policy Evaluation for Value-based Sequential Decision-making}

\author{%
  Alan Chan* \\
  Department of Computing Science\\
  University of Alberta\\
  \texttt{achan4@ualberta.ca} \\
   \And
   Kris De Asis* \\
   Department of Computing Science \\
   University of Alberta \\
   \texttt{kldeasis@ualberta.ca} \\
   \AND 
   Richard S. Sutton\\
   Department of Computing Science \\
   University of Alberta \\
   \texttt{rsutton@ualberta.ca} \\
}

\begin{document}

\maketitle

\begin{abstract}
Value-based methods for reinforcement learning lack generally applicable ways to derive behavior from a value function. Many approaches involve approximate value iteration (e.g., $Q$-learning), and acting greedily with respect to the estimates with an arbitrary degree of entropy to ensure that the state-space is sufficiently explored. Behavior based on explicit greedification assumes that the values reflect those of \textit{some} policy, over which the greedy policy will be an improvement. However, value-iteration can produce value functions that do not correspond to \textit{any} policy. This is especially relevant in the function-approximation regime, when the true value function can't be perfectly represented. In this work, we explore the use of \textit{inverse policy evaluation}, the process of solving for a likely policy given a value function, for deriving behavior from a value function. We provide theoretical and empirical results to show that inverse policy evaluation, combined with an approximate value iteration algorithm, is a feasible method for value-based control. 
\end{abstract}

\section{Value-based Control}
Value-based methods form an important class of reinforcement learning (RL) algorithms \citep{sutton2018reinforcement}. They estimate expected outcomes conditioned on a behavior policy and use such estimates to inform decision-making.

Q-learning \citep{watkins1992q} is a popular value-based RL control algorithm which enjoys convergence guarantees in the tabular setting \citep{bertsekas1996neuro}. It aims to estimate every situation's best expected outcome, and assuming that the estimates are accurate, behaving greedily with respect to them is optimal. However, especially with function approximation, the estimates often aren't accurate and may not reflect any achievable expected outcome in an environment \cite{dadashi2019value}.
Along this vein, \citet{lu2018non} recently investigated divergence due to \textit{delusional bias}, where the Q-learning update is blind to the function approximator's representable policies.

Being off-policy, i.e., estimating outcomes under the optimal policy using data collected from \textit{any} policy different from the optimal policy, Q-learning can use fixed behavior policies or learn from fixed batches of data. While this work is applicable in these scenarios, we emphasize our focus on the online setting with changing behavior policies, where an agent must balance estimation accuracy and control performance.

A common behavior policy used with online Q-learning, $\epsilon$-greedy \citep{sutton2018reinforcement}, suffers from several issues: over-exploitation of the estimated best action, sensitivity to the setting of $\epsilon$ (fixed or scheduled annealing), and non-smoothness of the policy with respect to the policy parameters. The latter may lead to convergence issues \citep{perkins2002existence,perkins2003convergent}. 
These inadequacies of $\epsilon$-greedy behavior, coupled with the aforementioned inconsistencies of Q-learning, motivate deriving behavior that is smooth with respect to the action-value estimates, and that approaches optimality as action-value estimates improve without having to use a schedule.
Under the Markov assumption, value functions satisfy a Bellman equation. While the equation is typically used to derive value estimation algorithms, one may fix the values and instead find a policy which satisfies the relationship, a process we refer to as \textit{inverse policy evaluation}.



We emphasize our intent of not trying to propose a state-of-the-art RL algorithm, but to develop, motivate, and understand inverse policy evaluation as a novel way to derive behavior from value estimates. Our specific contributions include: 1) introducing and analyzing inverse policy evaluation for deriving behavior from a value function, 2) providing an algorithm and noting its connection with policy gradient methods, 
3) proving convergence of algorithms under this framework to the optimal value function and optimal policy, and 4) providing empirical results which demonstrate key properties of our approach.

\section{Reinforcement Learning}
Reinforcement learning (RL) formalizes the sequential decision-making problem with the Markov Decision Process (MDP) framework \citep{puterman2014markov,sutton2018reinforcement}. 
At each discrete time step $t$, an agent observes the current state $S_t \in \statespace$, where $\statespace$ is the set of states in an MDP. Given $S_t$, 
an agent selects an action $A_t \in \actionspace(S_t)$ to perform in the environment, where $\actionspace(s)$ is the set of available actions in state $s$. The environment then returns a reward $R_{t+1} \in \mathbb{R}$, and an observation of the next state $S_{t+1} \in \statespace$, sampled from the environment's transition dynamics: $p(s', r|s, a) = \Pr(S_{t+1}=s', R_{t+1}=r|S_t=s, A_t=a)$. 
An agent selects actions according to a policy $\pi(a|s) = \Pr(A_t=a|S_t=s)$, and its goal is to find an \textit{optimal policy} $\pi^*$ which from each state, maximizes the expected return: a discounted sum of rewards with discount factor $\gamma$. 

\textit{Value-based} methods estimate \textit{value functions} which quantifies a policy's performance. A \textit{state-value} function represents the expected return from starting in state $s$ and following policy $\pi$: $\vpi(s) = \Ex_\pi\left[\sum_{k = 0}^\infty \gamma^k R_{t + k + 1} \mid S_t = s\right]$. An \textit{action-value} function is similar, but further conditioned on taking immediate action $a$ in state $s$: $\qpi(s, a) = \Ex_\pi\left[\sum_{k = 0}^\infty \gamma^k R_{t + k + 1} \mid S_t = s, A_t = a\right]$. Computing a policy's value function is known as \textit{policy evaluation}. Value-based methods then rely on \textit{policy improvement}, where acting greedily with respect to a value function conditioned on another policy is guaranteed to produce an improved policy. \textit{Policy iteration} \citep{bertsekas2011approximate} interleaves policy evaluation and improvement until an optimal policy is found. These approaches contrast \textit{policy gradient} methods which explicitly parameterize a policy, and updates the parameters to maximize an objective \citep{sutton2000policy, sutton2018reinforcement}. 


\section{Q-learning}
Q-learning \citep{watkins1992q} is a popular value-based method which tries to estimate the optimal value function (the optimal policy's value function) directly. Value functions can be expressed in terms of successor states' values through their \textit{Bellman equations}. For $\qpi$: $\qpi(s, a) = \Ex_{s', r}[r + \gamma \sum_{a'}{\pi(a'|s')\qpi(s', a')}]$.
When $\pi$ is greedy with respect to $Q_\pi$, the latter term becomes the maximum action-value in the next state, and we get the \textit{Bellman optimality equation}, the recursive relationship for the optimal action-value function $\qstar$. Q-learning updates action-values toward a sample-based evaluation of the Bellman optimality equation for $\qstar$:
\begin{align*}
    Q(S_t, A_t) &\leftarrow Q(S_t, A_t) + \alpha \big(R_{t+1} + \gamma \max_{a'}{Q(S_{t+1}, a')} - Q(S_t, A_t)\big)
\end{align*}
given some step size $\alpha$. By updating toward a combination of a sampled reward with the current estimate for the successor state's value, Q-learning belongs to the \textit{temporal difference} \citep{sutton1988td} family of value-based methods.

Convergence of $Q$ to $\qstar$ for every state-action pair requires that all state-action pairs are visited infinitely often \citep{bertsekas1996neuro}. Even if convergence is not guaranteed, as is the case with function approximation \citep{baird1995residual,tsitsiklis1997analysis,achiam2019towards}, it is still necessary for the agent to \textit{explore} different parts of state-space to obtain a reasonable estimate of $\qstar$ to inform decision making. 

This exploration requirement is often satisfied with $\epsilon$-greedy action selection \citep{mnih2013playing,sutton2018reinforcement}, where an agent behaves greedily with respect to its current estimates with probability $1 - \epsilon$, and behaves uniform randomly otherwise.
However, there are practical concerns with $\epsilon$-greedy behavior: 1) reliance on random exploration is likely inefficient in large state spaces \citep{korenkevych2019autoregressive}; 2) due to estimation error or representational capacities, the greedy action may be a poor choice \citep{hasselt2010double}; 3) needing to specify $\epsilon$ as a fixed value or an annealing schedule; and 4) the non-smoothness of the behavior policy with respect to changes in the value function can result in non-convergence \citep{perkins2002existence,perkins2003convergent,wagner2013optimistic}.

A notable alternative to $\epsilon$-greedy is a Boltzmann policy over the action-values: $\pi(\cdot|s) \propto \exp(Q(s,a)\tau^{-1})$, where $\tau > 0$ is a \textit{temperature} parameter controlling the policy's stochasticity, akin to $\epsilon$. However, such policies can be sensitive to $\tau$, more so than $\epsilon$ as the scale of Q varies considerably across problems. They may also chatter about multiple fixed-points when $\tau$ is fixed \cite{asadi2017alternative}.

Conservative policy iteration (CPI) addresses the non-smoothness issue by computing an exponential moving average of greedy policies, where the mixture rate must be carefully chosen to guarantee policy improvement \citep{kakade2002approximately}. \citet{vieillard2019deep} extends CPI to the neural network function approximation setting, emphasizing the benefit in regularizing greediness, while detailing the practical considerations for specifying the mixture rate. 


While many alternatives have been explored, many of them directly work with the policy that an agent aims to evaluate \textit{eventually}, i.e., some modification of the greedy policy. No work to our knowledge has explicitly considered the policy which corresponds with the \textit{current} values estimates.
Such an approach would preferably take estimation error into account, e.g., due to insufficient exploration or function approximation errors, rather than assume the current values are accurate. 
Given $Q = \qstar$, the policy that gives zero Bellman error \textit{is} an optimal policy by the uniqueness of the solution of the Bellman optimality equation \citep{bertsekas1996neuro}. Should $Q$ have estimation errors, the policy should still be suboptimal in a way that's directly related to the errors.

\section{Inverse Policy Evaluation}
All proofs may be found in the appendix. 

In policy evaluation \citep{sutton2018reinforcement}, a policy $\pi$ is fixed, and the goal is to estimate its value function. Let $\functionspace$ denote the space of functions available to approximate its value function. Policy evaluation is characterized by:  
\begin{align}\label{eq:problem}
    \min_{Q \in \functionspace}\, &\|Q(s, a) - (\bellman^\pi Q)(s, a)\| = \min_{Q \in \functionspace}\, \|(Q(s, a) - \Ex_{s', a'}[r(s, a) + \gamma \, Q(s', a')]\|, \nonumber
\end{align}
where $s' \sim p(\cdot|s, a)$ and $a' \sim \pi(\cdot|s')$, and $\bellman^\pi$ is the Bellman operator. There are several choices for the norm \citep{maei2011gradient}. 
Our proposed method, denoted \textit{inverse policy evaluation} (IPE), tries to derive a policy that is ``consistent'' with a given approximate action-value function in the following sense:
\begin{equation}\label{eq:eval-policy-problem}
    \pi \in \argmin_{\beta \in \Pi} \|Q(s, a) - (\bellman^\beta Q)(s, a)\|
\end{equation}
$\|\cdot\|$ is some fixed norm over state-action pairs. We call a policy $\pi$ so derived a value function's \textit{evaluation policy}. Of note, for a given $Q$, there \textit{may not exist} a policy $\pi$ such that $\qpi = Q$ \citep{dadashi2019value}. Nevertheless, it is still reasonable to consider such policies since we can quantify their performance, as we do now. 

\subsection{How does IPE account for estimation error?}
Let us solidify the intuition that the evaluation policy takes function approximation error into account. 

\begin{proposition}\label{prop:reasonable-ipe}
Assume that we are trying to estimate $\qpi$ with $Q$, for some $\qpi$. Denote the solution of \Cref{eq:eval-policy-problem} by $\ipepi$. Let $\|\cdot\|$ denote any norm under which Bellman operators are contraction mappings (e.g., infinity norm). We have the following bound.
\begin{align*}
    \|\qpi &- Q^{\ipepi}\| \leq \frac{1}{1 - \gamma}\left( (1 + \gamma) \|\qpi - Q\| + \|\bellman^{\ipepi} Q - Q\| \right)
\end{align*}
\end{proposition}
The first norm on the right-hand side of \Cref{prop:reasonable-ipe} measures the function approximation error of $Q$ in estimating $\qpi$. The second norm on the right-hand side is exactly the objective in \Cref{eq:eval-policy-problem}, minimized by $\ipepi$. Note that this bound is tight for $Q = \qpi$, for which both sides are 0. The upshot is that the return generated by $\ipepi$ is close to the return generated by $\pi$, proportional to how close $Q$ is to $\qpi$. By examining the proof of \Cref{prop:reasonable-ipe} in the Appendix, a similar statement also holds if we replace $Q$ by $V$. 
In particular, if we select $\qpi = \qstar$, then the difference between the return of $\ipepi$ and the optimal return is directly proportional to the function approximation error $\|Q - \qstar\|$, and is zero when $Q = \qstar$.

If we start from one evaluation policy and change our action-value estimates slightly, how much does the next evaluation policy change? The next proposition quantifies this relationship.

\begin{proposition}\label{prop:smooth-ipe}
Suppose $Q_1, Q_2$ are two action-value estimates, which may or may not be actual action-values. Let $\pi_i$ denote the evaluation policy of $Q_i$. Then 
\begin{align*}
    \|Q^{\pi_1} &- Q^{\pi_2}\| \leq \frac{1}{1 - \gamma}( (1 + \gamma) \|Q_1 - Q_2\|  + \|\bellman^{\pi_1}Q_1 - Q_1\| + \|\bellman^{\pi_2} Q_2 - Q_2\|  )
\end{align*}
\end{proposition}
This smoothness result for IPE is in contrast to $\epsilon$-greedy policies, which are known to suffer from non-smoothness in changes of the action-value function estimate \citep{perkins2002existence,perkins2003convergent}. In the presence of a small action gap, noise can cause an $\epsilon$-greedy policy to shift dramatically in quality. 


To estimate an evaluation policy, just as in policy evaluation, we can derive an upper bound of the objective function in Equation \ref{eq:eval-policy-problem} with Jensen's inequality, choosing $\|\cdot\|$ to be the $\ell_2$ norm:
\begin{align}\label{eq:upper-bound-eval-policy}
    \Ex_{s, a}&[(f(s, a) - \Ex_{s', a'}[r(s, a) + \gamma \, f(s', a')])^2]\leq \Ex_{s, a, s'}[(r(s, a) + \gamma \Ex_{a'}[Q(s', a')] - Q(s, a))^2]
\end{align}
We will minimize the inner term of the RHS. Let our policy $\pi_\policyparams$ be smoothly parameterized by $\policyparams$, with step-size $\alpha_t > 0$ at time $t$. Consider a transition $(S_t, A_t, R_{t+1}, S_{t+1})$. Let $\delta := R_{t+1} + \gamma \sum_{a'} \pi_\policyparams(a'|S_{t+1})Q(S_{t+1}, a') - Q(S_t, A_t)$ be the expected TD error. We then update with $\nabla_\policyparams (\delta)^2 $. The update with respect to $\policyparams$ is:
\begin{align}\label{eq:ipe-update}
    \policyparams_{t + 1} &\leftarrow \policyparams_t - \alpha_t 2\delta \gamma \sum_{a'} \nabla_\policyparams \pi_\policyparams(a'|S_{t+1}) Q(S_{t+1}, a')
\end{align}
We note in passing that \Cref{eq:ipe-update} is remarkably similar to the \textit{all-actions policy gradient} update \citep{sutton2000policy,sutton2018reinforcement}:
\begin{equation}
    \policyparams_{t + 1} \leftarrow \policyparams_t + \alpha_t \sum_a \nabla_\policyparams \pi_\policyparams(a|s) Q(s, a)
\end{equation}
One can interpret IPE as attempting to find a policy that matches the \textit{proposed returns} of an approximate value function $V$, with $\delta$ in \Cref{eq:ipe-update} changing signs appropriately to ensure that the return is matched, rather than maximized. There may be fruitful connections to training RL agents to achieve a \textit{specific return}, rather than on solely achieving the maximum return \citep{srivastava2019training}.

\section{IPE for Control}
The results in \Cref{prop:reasonable-ipe} and \Cref{prop:smooth-ipe} motivate a strategy for using IPE for control: given some approximate value iteration procedure for learning $Q^*$ (e.g., Q-learning), one can interleave updates of approximate value iteration with inverse policy evaluation. We denote such an interleaving procedure by VI-IPE (value iteration-inverse policy evaluation). We present one possible interleaving granularity in \Cref{alg:ipe}.

\begin{algorithm}[!htb]\small
\caption{\small Approximate Value Iteration-Inverse Policy Evaluation (VI-IPE)}
\label{alg:ipe}
\begin{algorithmic}
\State Given: action-value estimate $Q$, learned policy $\pi_\policyparams$.
\State $s \sim p(s_0)$
\State $a \sim \pi(\cdot|s)$
\State $t \leftarrow 0$
\While{$t \neq t_{max}$}
    \State Draw action $a$ according to $Q(s, \cdot)$ and/or $\pi_\policyparams(\cdot|s)$
    \State $s', r \sim p(s', r|s, a)$
    \State Perform one step of approximate value iteration (e.g., Q-learning) on $Q$
    \State Update $\policyparams$ according to \Cref{eq:ipe-update}.
    \State $s \leftarrow s'$
    \State $t \leftarrow t + 1$
\EndWhile
\end{algorithmic}
\end{algorithm}

We further derive some theoretical guarantees for VI-IPE in the tabular setting.
\begin{theorem}[VI-IPE]\label{thm:vipe}
Let $V_0 \in \R^{|\statespace|}$ and set $V_{i + 1} := \bellman^* V_i$ (value iteration). Write $\|\cdot\|$ for a norm for which $\bellman^*$ and $\bellman^\pi$ are contraction mappings (e.g., a weighted maximum norm) for all policies $\pi$. Let $\pi_k \in \argmin_\pi \|\bellman^\pi V_k - V_k\|$. Then 
\begin{align*}
    \|V^{\pi_k} - V^*\| &\leq \gamma^k \|V_1 - V_0\| \left( 1 + \frac{\gamma}{1 - \gamma} \right) + \gamma^k \|V_0 - V^*\|.
\end{align*}
\end{theorem}
This gives us \textit{monotonic} improvement in $V^{\pi_k}$ in terms of distance to $V^*$ as measured by by $\|\cdot\|$. We also have an approximate version of \Cref{thm:vipe}. 
\begin{theorem}[Approximate VI-IPE]\label{thm:approx-vipe}
Let $V_0 \in \R^{|\statespace|}$ and set $V_{i + 1} := \bellman^* V_i + \epsilon_{i + 1}$ (approximate value iteration). Write $\|\cdot\|$ for a norm for which $\bellman^*$ and $\bellman^\pi$ are contractions (e.g., a weighted maximum norm) for all policies $\pi$. Let $\pi_k \in \argmin_\pi \|\bellman^\pi V_k - V_k\|$. Then 
\begin{align*}
    \|V^{\pi_k}& - V^*\| \leq \left( 1 + \frac{\gamma}{1 - \gamma} \right)\Big(\gamma^k \|V_1 - V_0\| + \|\epsilon_k\| + \sum_{t = 1}^{k - 1} \gamma^t \|\epsilon_{k - t + 1} - \epsilon_{k - t}\| \Big)\\
    &\quad\quad + \gamma^k \|V_0 - V^*\| + \sum_{t = 0}^{k - 1} \gamma^t \|\epsilon_{k - t}\|.
\end{align*}
\end{theorem}

\section{IPE for Hyper-parameter Selection}\label{section:hyperparam}

When might an evaluation policy remain deterministic, or close to it? Let us imagine for the moment that $\pi$ is a tabular softmax policy with one logit for each action. If $\pi$ is deterministic at a state $s$, then the gradient of $\pi$ with respect to its logits is zero. If that were the only term in our update in \Cref{eq:ipe-update}, a deterministic policy would remain deterministic. 
However, there is another term in our update, $\delta$. If $\delta$ is large, meaning that the action value estimate is poor, the resulting gradient may be large and $\pi$ will not remain deterministic. Instead, $\pi$ will be modified to bring the TD error $\delta$ as close to zero as possible. 


As such, the entropy of the evaluation policy may reflect some level of confidence in the estimation accuracy. While one might suggest directly using the $\delta$ term as a measure of estimation accuracy, having it expressed as entropy makes it relatively straightforward to map to hyper-parameters with probabilistic interpretations, e.g., $\lambda$ in TD($\lambda$) \cite{sutton1988td}, $\gamma$ in the return, $\epsilon$ in $\epsilon$-greedy.

IPE does not explicitly account for exploration, which is essential in learning good policies. To fix this shortcoming, we also try to use IPE to select a degree of randomness in the behaviour policy. 
In this work, we consider IPE for adaptively setting $\epsilon$ in an $\epsilon$-greedy behavior policy. We do so with an \textit{entropy matching} procedure where $\epsilon$ will be chosen such that the resulting $\epsilon$-greedy policy matches the entropy of the evaluation policy in a given state. We denote this method $\epsilon$-IPE, and emphasize that when accounting for IPE's step size, this does not increase the number of parameters, and rather \textit{decreases} it when $\epsilon$ would have otherwise followed an annealing schedule.

\section{Empirical Evaluation}

\subsection{2-state Switch-Stay MDP}

We visualize the dynamics of $\epsilon$-greedy, IPE, and $\epsilon$-IPE on a simple, 2-state ``Switch-Stay'' MDP:

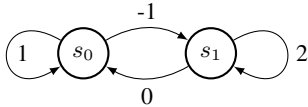
\begin{figure}[!htb]
\begin{center}
\begin{tikzpicture}[auto,node distance=10mm,>=latex,font=\small]
    \tikzstyle{round}=[thick,draw=black,circle]

    \node[round] (s0) {$s_0$};
    \node[round,right=10mm of s0] (s1) {$s_1$};

    \draw[->] (s0) to [out=30, in = 150] node {-1} (s1);
    \draw [->] (s0) to [out=150,in=210,loop] node {1} (s0);
    \draw[->] (s1) to [out=210, in=330]  node{0}(s0);
    \draw [->] (s1) to [out=30,in=330,loop] node{2} (s1);
\end{tikzpicture}
\end{center}
    \caption{The switch-stay MDP. All transitions are deterministic and the agent starts in state $s_0$.}
    \label{fig:switch-stay-mdp}
\end{figure}

\textbf{Question 1: Can IPE be used for control?}
First, we visualize the progress of the respective policies on the Switch-Stay MDP's value function polytope \citep{dadashi2019value}. We evaluate the \textit{online} performance of each algorithm for a representative run of 500 steps. Blue dots represent Q-learning's value estimates (state-values computed from the learned action-values) at each step, and green dots represent the true values of a time step's current behavior policy. Red circles highlight the initial points of each trajectory.

We compared $\epsilon$-greedy with a fixed $\epsilon=0.1$, $\epsilon$-greedy where $\epsilon$ linearly decays $\epsilon$ from $1.0$ to $0.1$ over 100 steps, using an estimated evaluation policy (labeled IPE), as well as $\epsilon$-IPE. Each agent used a Q-learning step size of $\alpha_Q = 0.5$, and the IPE variants used a policy step size of $\alpha_\pi = 0.05$. Such parameter settings were chosen to be representative of each behavior policy's dynamics.

\begin{figure}[!h]
    \centering
    \begin{subfigure}[b]{0.4\linewidth}
        \centering
        \includegraphics[width=1.0\columnwidth]{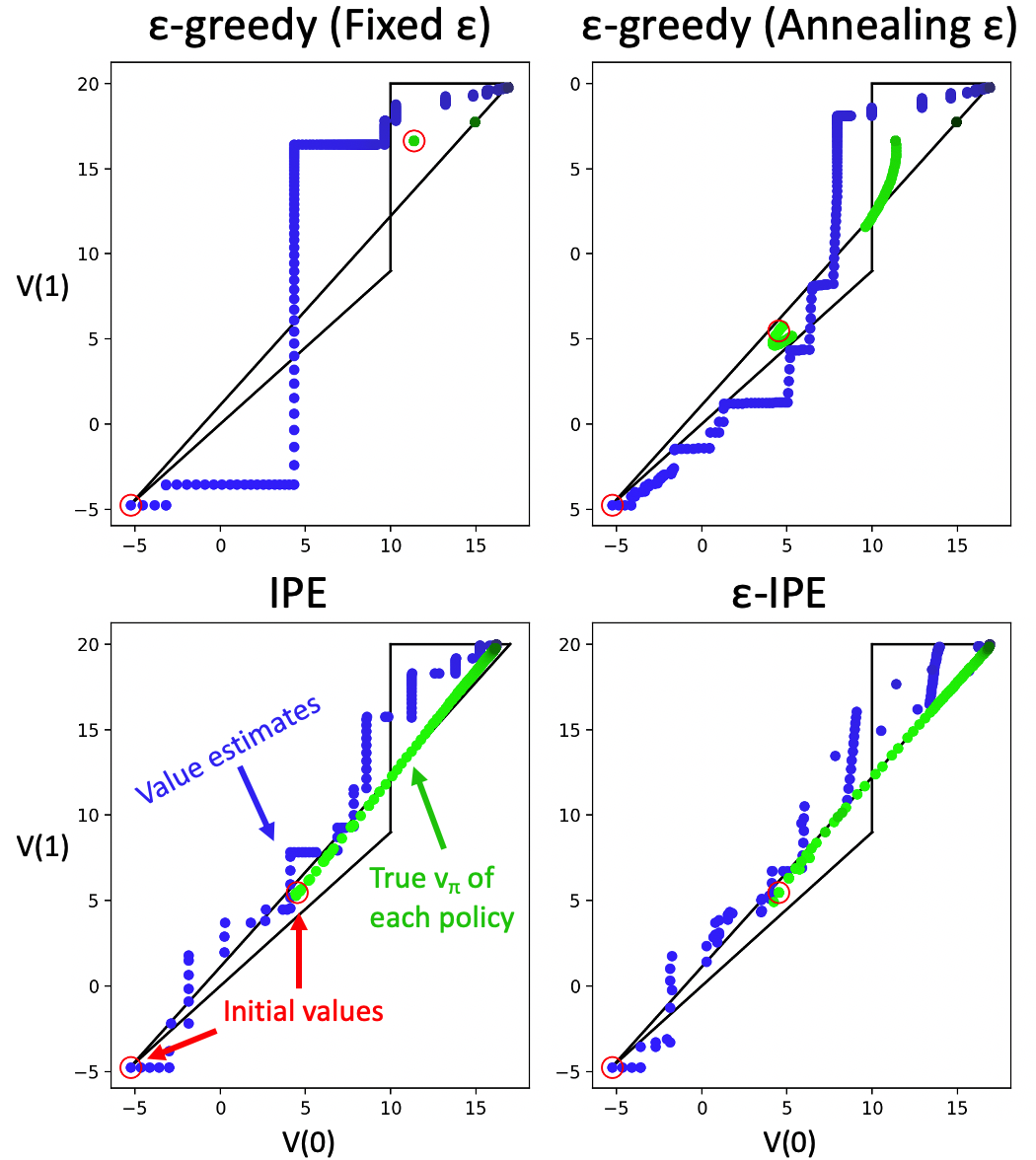}
       \caption{}\label{fig:polytope}
    \end{subfigure}\hspace{30pt}%
   \begin{subfigure}[b]{0.4\linewidth}
        \centering
        \includegraphics[width=1.0\textwidth]{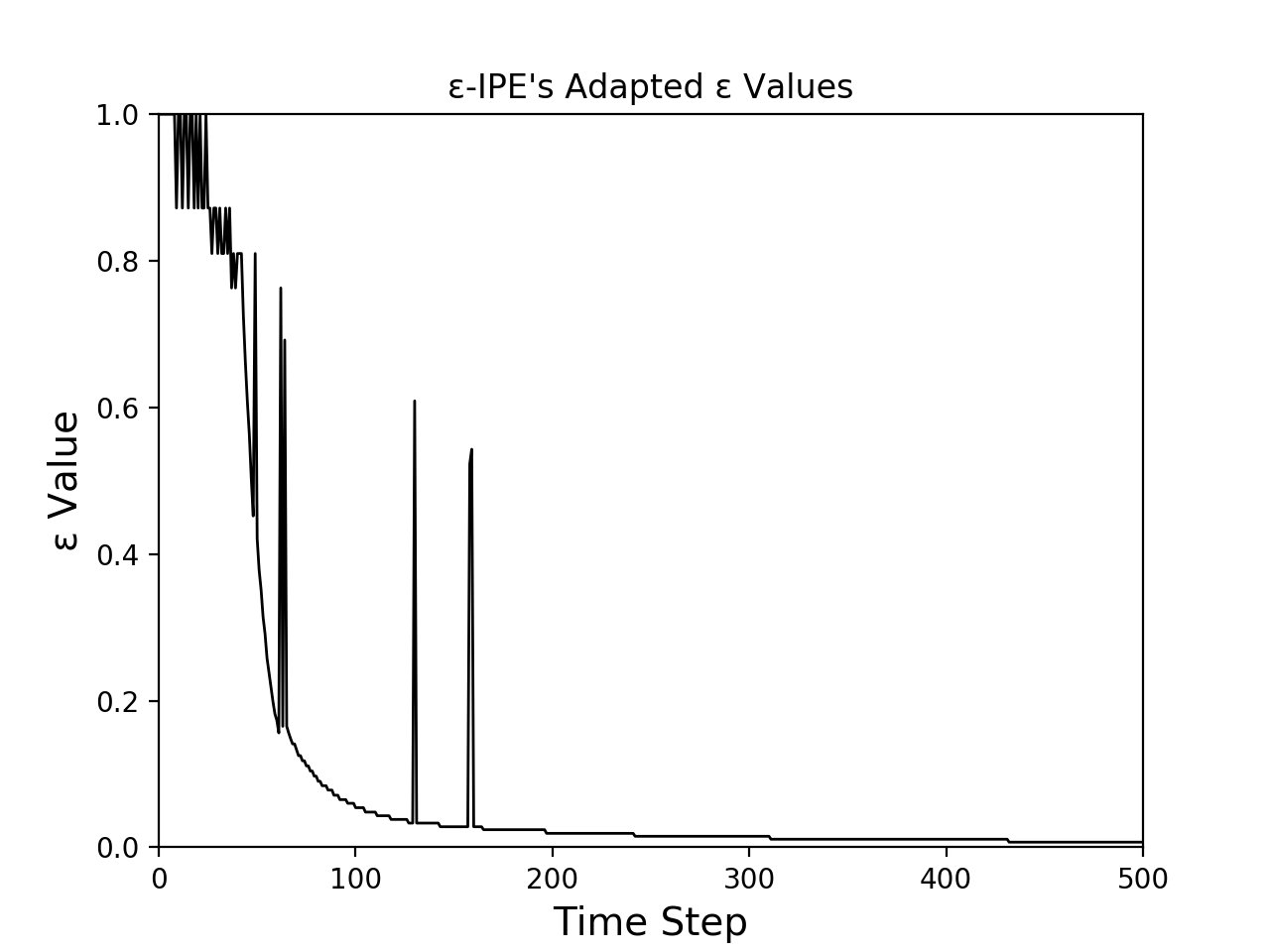}
        \caption{}
        \label{fig:adapteps}
    \end{subfigure}
    \caption{\textbf{a)} Learning dynamics on the Switch-Stay MDP's value function polytope. Blue trajectories reflect Q-learning's value estimates, while green ones reflect the true value function of the policy at each step. \textbf{b)} $\epsilon$ values adapted by an $\epsilon$-IPE behavior policy. The apparent value discretization is due to a pre-computed lookup-table approach for selecting an entropy-equivalent $\epsilon$ value.}
\end{figure}

Q-learning's value estimates can leave the polytope: i.e., the values don't correspond to any policy. Such behavior is expected of value-iteration, and is related to delusional bias \cite{dadashi2019value, lu2018non}. $\epsilon$-greedy with a fixed, small $\epsilon$ suffered more from this problem, suggesting that it may be related to drastic changes in the data distribution induced by a changing policy. Linearly annealing $\epsilon$ produces arc-shaped paths in the behavior policy values, indicative of a linear interpolation between equiprobable and greedy policies. In both the fixed and linearly annealing $\epsilon$ cases, we notice large jumps in the behavior policy values, representing switches in the value estimates' greedy actions. 

With the IPE approaches, the values of the behavior policy followed a much smoother path. While this might be expected of IPE given its policy-gradient-like update, perhaps surprisingly, adapting an $\epsilon$-greedy policy based on IPE exhibited comparable smoothness. To understand how $\epsilon$ was adapted over time, Figure \ref{fig:adapteps} shows the $\epsilon$ used at each step. A seemingly sigmoidal annealing schedule was adapted, contrasting the annealing schedules typically used in the literature \citep{mnih2013playing}.

\textbf{Question 2: How deterministic is an evaluation policy?}
To develop intuition for an evaluation policy's determinism, we swept a range of fixed value functions- pairs of $V(0) \in \{-6, -4, -2 ,\cdots , 18\}$, and $V(1) \in \{-6, -4, \cdots, 22\}$. We solve for each fixed value function's evaluation policy, compute the true value function of the evaluation policy, and show how a value function gets mapped back to the polytope in terms of the derived behavior policy. This is represented as an arrow from the fixed value function to the behavior policy's value function in the polytope. For comparison, the procedure was repeated for a greedy policy. The resulting value maps are visualized in Figure \ref{fig:value-map}. 

Most interesting is how the fixed value functions \textit{outside} of the polytope get mapped back into it. With the evaluation policies, fixed value functions that are relatively near to the polytope appear to get mapped to a nearby, non-deterministic point of the polytope. If the values are far from the polytope, they can map to a deterministic policy. A possible explanation is that if an action-value is dramatically larger than another, the all-actions policy gradient term in Equation \ref{eq:ipe-update} 
dominates and greedifies toward the large estimate. Such extreme cases seem unlikely based on the trajectories observed in Figure \ref{fig:polytope}, as it would require value-iteration to move considerably in an orthogonal direction. Taken together, this supports our analysis of an evaluation policy's determinism (Section \ref{section:hyperparam}) in terms of $\delta$ and the all-actions policy gradient, suggesting that it tend to be stochastic for reasonable deviations from the polytope.

\begin{figure}[!h]
    \centering
    \begin{subfigure}[b]{0.49\textwidth}
        \includegraphics[width=\textwidth]{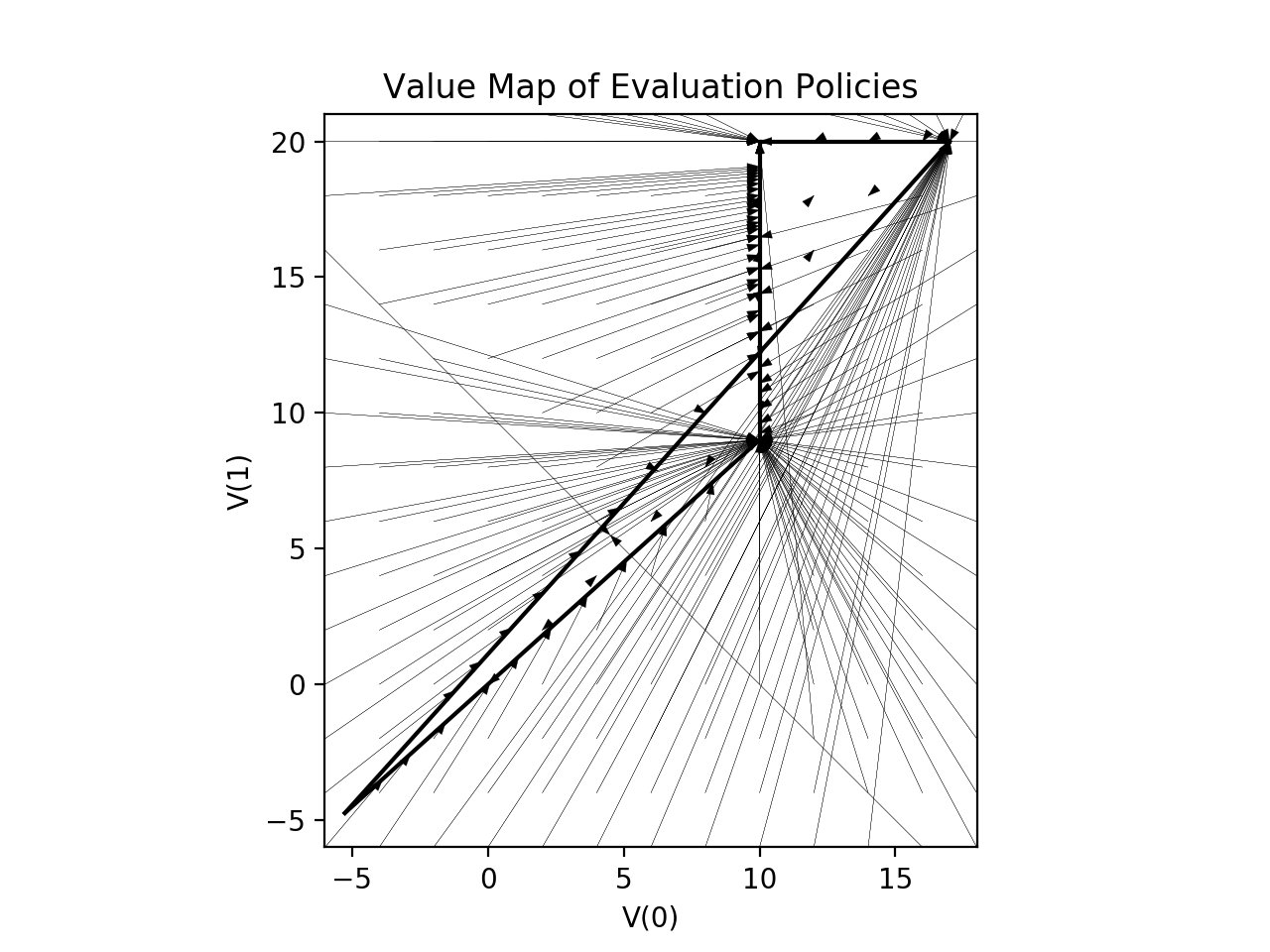}
        \label{fig:value-map-ipe}
    \end{subfigure}
    \begin{subfigure}[b]{0.49\textwidth}
        \includegraphics[width=\textwidth]{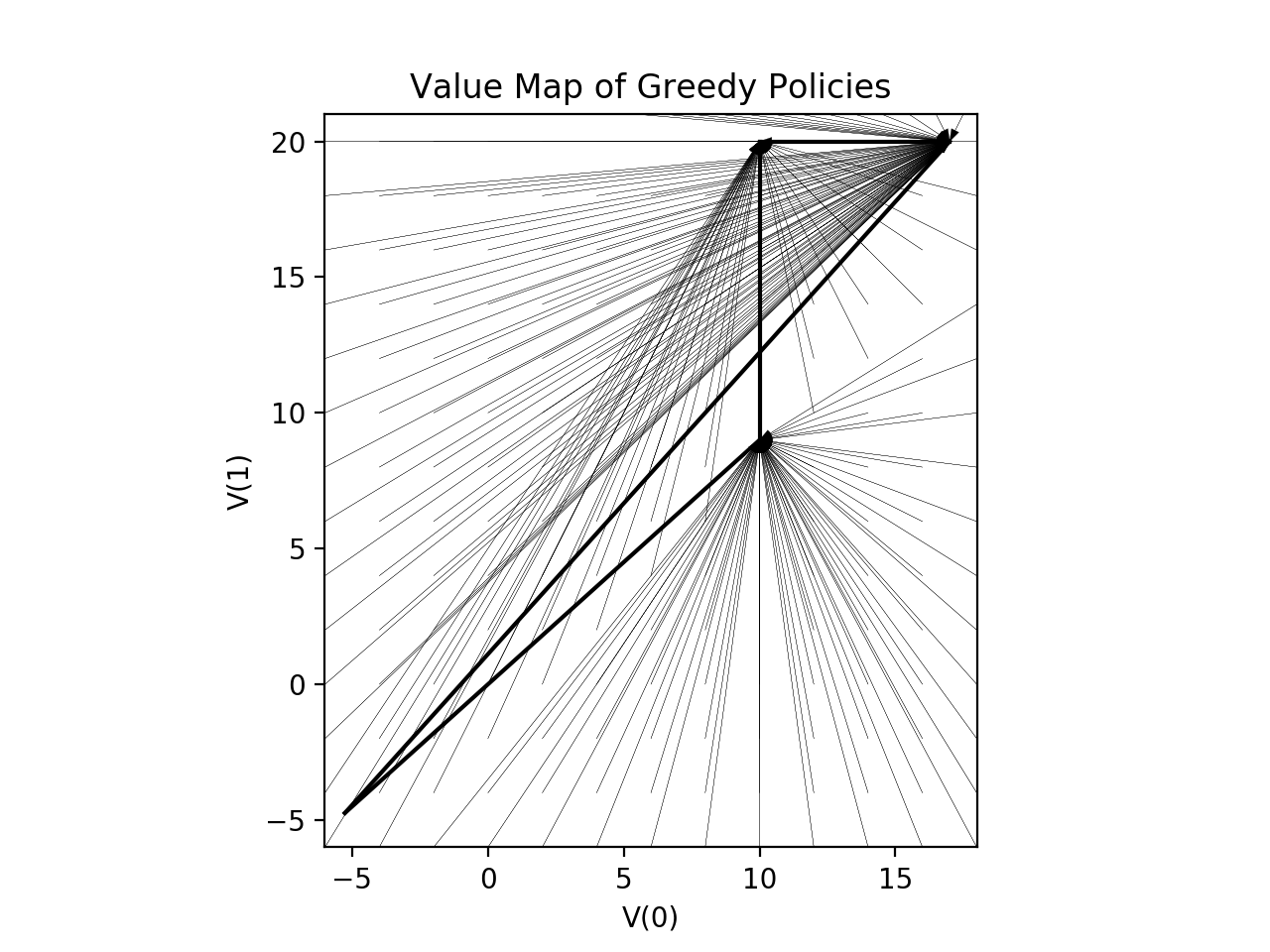}
        \label{fig:value-map-eps-greedy}
    \end{subfigure}
    \caption{A visualization of how value functions are mapped back to the value function polytope through derived behavior policies. }\label{fig:value-map}
\end{figure}


\textbf{Question 3: How sensitive is IPE to its hyperparameters in comparison with $\epsilon$-greedy?}
We varied $\epsilon$ in $\epsilon$-greedy with fixed $\epsilon$, the number of steps for an annealing $\epsilon$ to linearly anneal from $1.0$ to $0.1$, and the policy step size $\alpha_\pi$ for IPE and $\epsilon$-IPE. Each setting performed 1000 runs of 500 steps, and Figure \ref{fig:switch-stay-reward} shows the average reward over the 500 steps, as well as the final root-mean-squared error (RMSE) in the approximated optimal value function.

\begin{figure}[!h]
    \centering
    \includegraphics[width=0.8\textwidth]{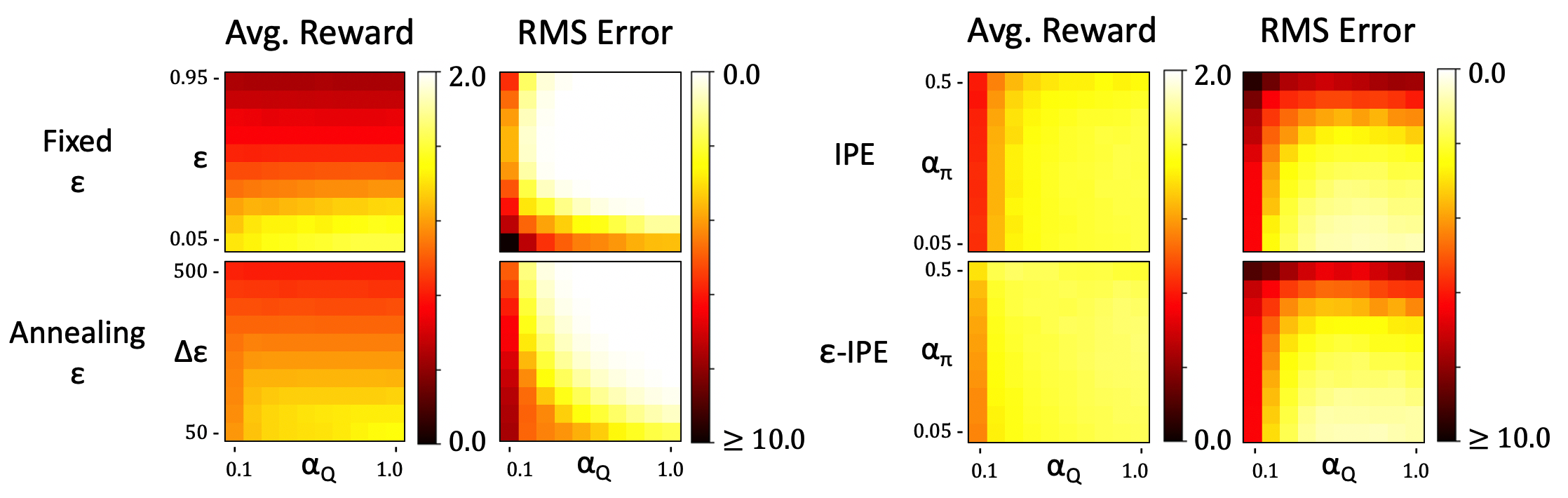}
    \caption{Hyperparameter sensitivities on the switch-stay MDP. 1000 runs of 500 steps were performed for each behavior policy's hyperparameter configuration.}\label{fig:switch-stay-reward}
\end{figure}


In \Cref{fig:switch-stay-reward}, Q-learning with $\epsilon$-greedy is quite sensitive to $\epsilon$. 
In contrast, both IPE and $\epsilon$-IPE learn the optimal policy across a wide range of learning rates. From looking at the $\epsilon$-greedy plots, both variants of $\epsilon$-greedy exhibited a negative correlation between the average reward and the value function RMSE. To attain high average reward, the agent tends to settle for an \textit{inaccurate} value function. This relationship might be due to overexploitation, and not seeing enough data to learn accurate action-values; on the other hand, a high $\epsilon$ leads to worse control performance from overexploration. 

On the other hand, IPE and $\epsilon$-IPE exhibit a positive correlation between the accuracy of the value function and the average reward obtained. Given the large overlap between the regions of high average reward and low RMSE, there seems to be, without careful parameter tuning, a natural adequate balance of (1) exploration needed to learn an accurate value function and (2) exploitation of value estimates to achieve a large expected return.

\subsection{Deep $\epsilon$-IPE}
We test $\epsilon$-IPE with function approximation on LunarLander-v2 \citep{brockman2016openai} and Freeway from the MinAtar suite \citep{young19minatar}. We note that $\epsilon$-IPE appeared much more consistent than behaving directly with the estimated evaluation policy. 

We use DQN \citep{watkins1992q,mnih2013playing} as our approximate value-iteration algorithm, and compare against $\epsilon$-greedy with a linear annealing schedule, as commonly used with DQN. We use the RMSprop optimizer \citep{tieleman2012lecture}, and sweep over the relevant hyperparameters for both DQN and $\epsilon$-IPE. Complete experimental details are in the appendix. 


Figure \ref{fig:lunar-lander} shows learning curves from DQN with $\epsilon$-IPE and, and DQN with $\epsilon$-greedy behavior over 500k frames of interaction in LunarLander-v2, averaged over 30 independent runs. For each behavior policy, the hyperparameter settings were based on the largest area under the curve among those tested.
DQN with $\epsilon$-greedy suffers from a large, consistent dip which lines up with when $\epsilon$ anneals to its final value. On the other hand, IPE enjoys relatively monotonic improvement, seemingly lower variance, better asymptotic performance, and a larger area under the curve. Figure \ref{fig:freeway} shows a similar comparison between the two behavior policies, but over 2M frames of interaction in Freeway. The areas under the curve were not significantly different in this domain, but $\epsilon$-IPE seemed to achieve a better asymptotic performance.


\begin{figure}[!htb]
    \centering
    \begin{subfigure}[b]{0.5\linewidth}
        \centering
        \includegraphics[width=\textwidth]{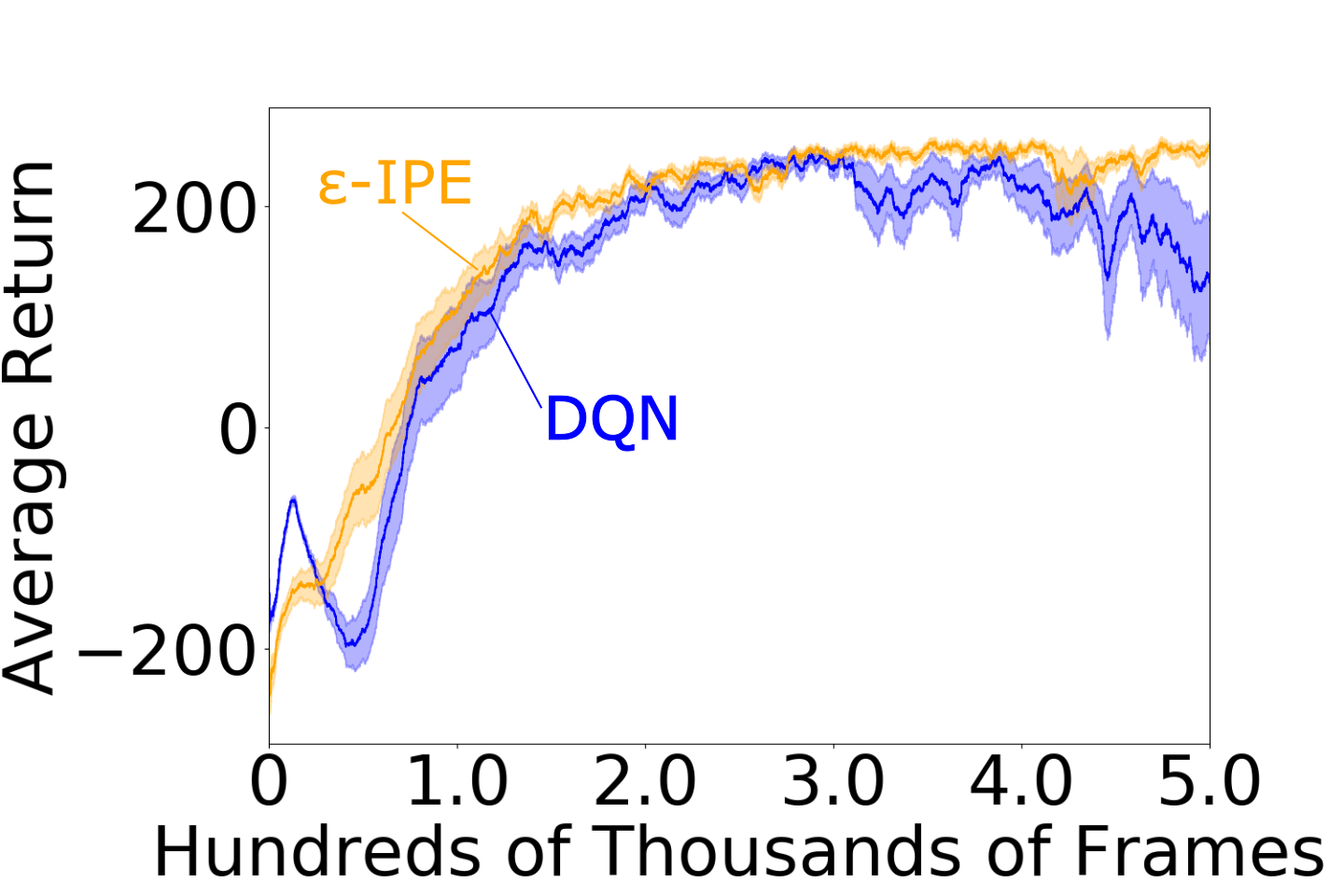}
        \caption{Lunar Lander}
        \label{fig:lunar-lander}
    \end{subfigure}%
    \begin{subfigure}[b]{0.5\linewidth}
        \centering
        \includegraphics[width=\textwidth]{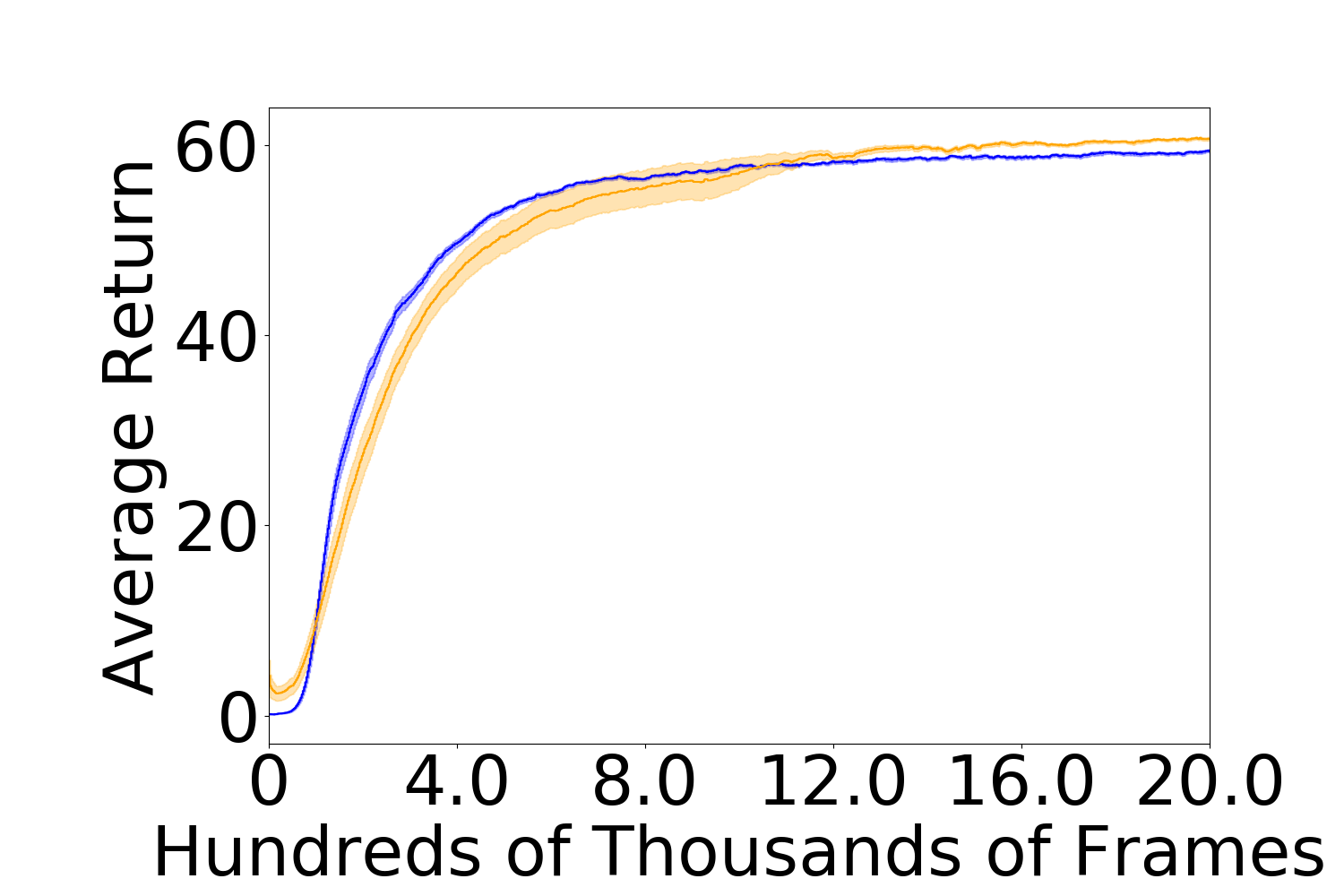}
        \caption{Freeway}
        \label{fig:freeway}
    \end{subfigure}
    \caption{Error bars represent standard error. Each point on the $y$-axis is the average return over the past 20 episodes, averaged over 30 runs.}
\end{figure}


\section{Discussion and Conclusion}
Our results suggest that IPE, when combined with an approximate value iteration algorithm, provides a novel, viable way to derive a sensible behavior policy for value-based control. We showed theoretically and empirically 1) that a learned evaluation policy can approach an optimal policy, 2) that an evaluation policy can maintain stochasticity in the face of value function estimation error, 3) that IPE can be less sensitive than $\epsilon$-greedy to the relevant hyperparameters, and 4) that an IPE-based policy can be competitive with $\epsilon$-greedy with DQN in a deep RL control task. 

This opens many avenues of future work. Different objectives can be formed for the general IPE problem, including the use of absolute Bellman errors, exploring multi-step returns, and dual formulations to optimize the true objective instead of an upper bound. When viewed from a policy gradient lens, it would be good to explore applications in batch, off-policy policy optimization, as well as the inclusion of policy gradient techniques, e.g., trust regions \citep{schulman2015trust}, baselines, and entropy regularization. Extensive evaluation with nonlinear function approximation would be good to better assess the method's scalability. 


\section*{Broader Impact}
The application of any RL agent should be subservient to legal and ethical obligations. One such obligation is predictability and explainability of the agent. The prevalent use of randomness in RL, and the instability of RL agents to randomness, is a barrier to these desiderata. Indeed, random exploration as used in $\epsilon$-greedy can lead harmful real-world impacts, like destruction of property or physical harm to living beings. Being able to control the level of randomness - and the consequences of that randomness - while still maintaining quality performance is essential. With IPE, we take a step towards the goal of reducing the need for completely random exploration in RL. 


\begin{ack}

\end{ack}

\bibliography{neurips_2020}

\section{Proofs}
\begin{proposition}\label{prop:reasonable-ipe}
Assume that we are trying to estimate $\qpi$ with $Q$, for some $\qpi$. Denote the solution of the IPE problem from $Q$ by $\ipepi$. Let $\|\cdot\|$ denote any norm under which Bellman operators are contraction mappings (e.g., infinity norm). We have the following bound.
\begin{align*}
    \|\qpi &- Q_{\ipepi}\| \leq \frac{1}{1 - \gamma}\left( (1 + \gamma) \|\qpi - Q\| + \|\bellman^{\ipepi} Q - Q\| \right)
\end{align*}
\end{proposition}
\begin{proof}
\begin{align*}
    \|Q^\pi - Q^{\ipepi}\| &\leq \|Q^\pi - Q\| + \|Q^{\ipepi} - Q\|\\
    &\leq \|Q^\pi - Q\| + \|Q^{\ipepi} - \bellman^{\ipepi} Q\| + \|\bellman^{\ipepi} Q - Q\|\\
    &\leq \|Q^\pi - Q\| + \gamma \|Q^{\ipepi} - Q\| + \|\bellman^{\ipepi} Q - Q\|\\
    &\leq \|Q^\pi - Q\| + \gamma (\|Q^{\ipepi} - Q^\pi\| + \|Q^\pi - Q\|) + \|\bellman^{\ipepi} Q - Q\|\\
\end{align*}
The first two inequalities are applications of the triangle inequality and the third inequality is from applying the facts that (1) Bellman operators are contraction mappings and that (2) $Q^{\ipepi}$ is a fixed point of $\bellman^{\ipepi}$. Rearranging, we have,
\begin{align*}
    \implies (1 - \gamma) \|Q^\pi - Q^{\ipepi}\| &\leq \|Q^\pi - Q\| + \gamma \|Q^\pi - Q\| + \|\bellman^{\ipepi} Q - Q\|\\
    \implies \|Q^\pi - Q^{\ipepi}\| &\leq \frac{1}{1 - \gamma} ((1 + \gamma) \|Q^\pi - Q\| + \|\bellman^{\ipepi} Q - Q\|).
\end{align*}
\end{proof}
The first norm on the RHS is the function approximation error of $Q$ in estimating $Q^\pi$. The second term is the IPE error, which we assume is minimized by $\ipepi$. The bound is tight since it is equal to 0 for $Q = Q^\pi$. Note that all the inequalities above still hold if we replace $Q$ by $Q$. 

If instead we have an arbitrary $Q$ and want to know what $Q^\pi$ is closest to $Q^{\ipepi}$, the bound above suggests that $Q^\pi$ must minimize $\|Q^\pi - Q\|$. 

Now, we want to say something about the smoothness of the IPE policy as our action-value estimates change. 

\begin{proposition}\label{prop:smooth-ipe}
Suppose $Q_1, Q_2$ are two action-value estimates, which may or may not be actual action-values. Let $\pi_i$ denote the evaluation policy of $Q_i$. Then 
\begin{align*}
    \|Q^{\pi_1} - Q^{\pi_2}\| \leq \frac{1}{1 - \gamma}\left( (1 + \gamma) \|Q_1 - Q_2\|  + \|\bellman^{\pi_1}Q_1 - Q_1\| + \|\bellman^{\pi_2} Q_2 - Q_2\|  \right)
\end{align*}
\end{proposition}
\begin{proof}
Let $Q$ be any action-value estimate. From the triangle, inequality, we have the following.
\begin{align*}
    \|Q^{\pi_1} - Q^{\pi_2}\| &\leq \|Q^{\pi_1} - Q\| + \|Q - Q^{\pi_2}\|
\end{align*}
Now, note that the proof of \Cref{prop:reasonable-ipe} does not actually use the fact that $Q^\pi$ is the action value of any policy. Hence, applying \Cref{prop:reasonable-ipe} twice with $Q^\pi = Q$, we derive
\begin{align*}
    \|Q^{\pi_1} - Q^{\pi_2}\| &\leq \frac{1}{1 - \gamma}\left( (1 + \gamma) \left(\|Q - Q_1\| + \|Q - Q_2\|  \right) + \|\bellman^{\pi_1}Q_1 - Q_1\| + \|\bellman^{\pi_2} Q_2 - Q_2\|  \right).
\end{align*}
Setting $Q = Q_1$, we finally have
\begin{align*}
    \|Q^{\pi_1} - Q^{\pi_2}\| &\leq \frac{1}{1 - \gamma}\left( (1 + \gamma) \|Q_1 - Q_2\|  + \|\bellman^{\pi_1}Q_1 - Q_1\| + \|\bellman^{\pi_2} Q_2 - Q_2\|  \right).
\end{align*}
\end{proof}
We note that this bound is not tight since setting $Q_1 = Q_2$ does not make the RHS 0. Nevertheless, if $Q_1 = Q^{\pi_1}$ (i.e., $Q_1$ is actually an action-value function), then setting $Q_2 = Q_1$ does make both sides of the inequality 0. 

\begin{theorem}[VI-IPE]\label{thm:vipe}
Let $V_0 \in \R^{|\statespace|}$ and set $V_{i + 1} := \bellman^* V_i$ (value iteration). Write $\|\cdot\|$ for a norm for which $\bellman^*$ and $\bellman^\pi$ are contractions (e.g., a weighted maximum norm) for all policies $\pi$. Let $\pi_k \in \argmin_\pi \|\bellman^\pi V_k - V_k\|$. Then 
\begin{align*}
    \|V^{\pi_k} - V^*\| \leq  \gamma^k \|V_1 - V_0\| \left( 1 + \frac{\gamma}{1 - \gamma} \right)  + \gamma^k \|V_0 - V^*\|.
\end{align*}
\end{theorem}
\begin{proof}
First, we have the following decomposition. 
\begin{align*}
    \|V^{\pi_k} - V^*\| &\leq \|V^{\pi_k} - \bellman^{\pi_k}V_k\| + \|\bellman^{\pi_k} V_k - V^*\|\\
    &\leq \|V^{\pi_k} - \bellman^{\pi_k}V_k\| + \|\bellman^{\pi_k} V_k - V_k\| + \|V_k - V^*\|
\end{align*}
Let us examine each term separately. The last term above can be bounded with the usual value iteration inequality. 
\begin{align*}
    \|V_k - V^*\| &= \|\bellman^* V_{k - 1} - \bellman^* V^*\|\\
    &\leq \gamma \|V_{k - 1} - V^*\|\\
    &\leq \gamma^k \|V_0 - V^*\|
\end{align*}
By definition of $\pi_k$ as an evaluation policy of $V_k$, we have
\begin{align*}
    \|\bellman^{\pi_k}V_k - V_k\| \leq \|\bellman^* V_k - V_k\|\\
    &= \|\bellman^* V_k - \bellman^* V_{k - 1}\|\\
    &\leq \gamma \|V_k  - V_{k - 1}\|\\
    &= \gamma \|\bellman^* V_{k - 1} - \bellman^* V_{k - 2}\|\\
    &\leq \gamma^2 \|V_{k - 1} - V_{k - 2}\|\\
    &\quad\vdots\\
    &\leq \gamma^k \|V_1 - V_0\|
\end{align*}
Time for the other term. 
\begin{align*}
    \|V^{\pi_k} - \bellman^{\pi_k}V_k\| &= \|\bellman^{\pi_k} V^{\pi_k} - \bellman^{\pi_k} V_k\|\\
    &\leq \gamma \|V^{\pi_k} - V_k\|\\
    &\leq \gamma \|V^{\pi_k} - \bellman^{\pi_k} V_k\| + \gamma \|\bellman^{\pi_k} V_k - V_k\|\\
    &\leq \gamma \|V^{\pi_k} - \bellman^{\pi_k} V_k\| + \gamma^{k + 1} \|V_1 - V_0\|\\
    \therefore \|V^{\pi_k} - \bellman^{\pi_k}V_k\| &\leq \frac{\gamma^{k + 1}}{1 - \gamma} \|V_1 - V_0\|.
\end{align*}
The first line comes from the Bellman equation for $V^{\pi_k}$. The fourth line comes from the inequality we just showed for $ \|\bellman^{\pi_k}V_k - V_k\|$. Putting everything together, we have 
\begin{align*}
    \|V^{\pi_k} - V^*\| &\leq \gamma^k \|V_1 - V_0\| \left( 1 + \frac{\gamma}{1 - \gamma} \right)  + \gamma^k \|V_0 - V^*\|
\end{align*}
\end{proof}
Note that \Cref{thm:vipe} gives us \textit{monotonic} improvement in $V^{\pi_k}$ in terms of distance to $V^*$ as measured by by $\|\cdot\|$. This is in contrast to $\epsilon$-greedy which enjoys no such guarantee, as we show in our experiments. 

\begin{theorem}[Approximate VI-IPE]\label{thm:approx-vipe}
Let $V_0 \in \R^{|\statespace|}$ and set $V_{i + 1} := \bellman^* V_i + \epsilon_{i + 1}$ (approximate value iteration). Write $\|\cdot\|$ for a norm for which $\bellman^*$ and $\bellman^\pi$ are contractions (e.g., a weighted maximum norm) for all policies $\pi$. Let $\pi_k \in \argmin_\pi \|\bellman^\pi V_k - V_k\|$. Then 
\begin{align*}
    \|V^{\pi_k} - V^*\| \leq \left( 1 + \frac{\gamma}{1 - \gamma} \right)\left(\gamma^k \|V_1 - V_0\| + \|\epsilon_k\| + \sum_{t = 1}^{k - 1} \gamma^t \|\epsilon_{k - t + 1} - \epsilon_{k - t}\| \right) + \gamma^k \|V_0 - V^*\| + \sum_{t = 0}^{k - 1} \gamma^t \|\epsilon_{k - t}\|.
\end{align*}
\end{theorem}
\begin{proof}
We have the following decomposition again. 
\begin{align*}
    \|V^{\pi_k} - V^*\| &\leq \|V^{\pi_k} - \bellman^{\pi_k}V_k\| + \|\bellman^{\pi_k} V_k - V^*\|\\
    &\leq \|V^{\pi_k} - \bellman^{\pi_k}V_k\| + \|\bellman^{\pi_k} V_k - V_k\| + \|V_k - V^*\|
\end{align*}
Let us examine each term separately. The bound for the last term is, again, from the usual value iteration result. 
\begin{align*}
    \|V_k - V^*\| &\leq \|\epsilon_k\| + \gamma\|V_{k - 1} - V^*\|\\
    &\leq \gamma^k \|V_0 - V^*\| + \sum_{t = 0}^{k - 1} \gamma^t \|\epsilon_{k - t}\|
\end{align*}
By definition of $\pi_k$ as an evaluation policy of $V_k$, we have
\begin{align*}
    \|\bellman^{\pi_k}V_k - V_k\| \leq \|\bellman^* V_k - V_k\|\\
    &= \|\bellman^* V_k - \bellman^* V_{k - 1} - \epsilon_k\|\\
    &\leq \gamma \|V_k  - V_{k - 1}\| + \|\epsilon_k\|\\
    &\leq \gamma \|\bellman^* V_{k - 1} - \bellman^* V_{k - 2} + \gamma \|\epsilon_k - \epsilon_{k - 1}\| + \|\epsilon_k\|\\
    &\leq \gamma^2 \|V_{k - 1} - V_{k - 2}\| + \gamma \|\epsilon_k - \epsilon_{k - 1}\| + \|\epsilon_k\|\\
    &\quad\vdots\\
    &\leq \gamma^k \|V_1 - V_0\| + \|\epsilon_k\| + \sum_{t = 1}^{k - 1} \gamma^t \|\epsilon_{k - t + 1} - \epsilon_{k - t}\| 
\end{align*}
Time for the first term. 
\begin{align*}
    \|V^{\pi_k} - \bellman^{\pi_k}V_k\| &= \|\bellman^{\pi_k} V^{\pi_k} - \bellman^{\pi_k} V_k\|\\
    &\leq \gamma \|V^{\pi_k} - V_k\|\\
    &\leq \gamma \|V^{\pi_k} - \bellman^{\pi_k} V_k\| + \gamma \|\bellman^{\pi_k} V_k - V_k\|\\
    &\leq \gamma \|V^{\pi_k} - \bellman^{\pi_k} V_k\| + \gamma^{k + 1} \|V_1 - V_0\| + \gamma\|\epsilon_k\| + \gamma \sum_{t = 1}^{k - 1} \gamma^t \|\epsilon_{k - t + 1} - \epsilon_{k - t}\| \\
    \therefore \|V^{\pi_k} - \bellman^{\pi_k}V_k\| &\leq \frac{1}{1 - \gamma} \left(\gamma^{k + 1} \|V_1 - V_0\| + \gamma\|\epsilon_k\| + \gamma \sum_{t = 1}^{k - 1} \gamma^t \|\epsilon_{k - t + 1} - \epsilon_{k - t}\|\right)
\end{align*}
The first line comes from the Bellman equation for $V^{\pi_k}$. The fourth line comes from the inequality we just showed for $ \|\bellman^{\pi_k}V_k - V_k\|$. Putting everything together, we have 
\begin{align*}
    \|V^{\pi_k} - V^*\| &\leq \left( 1 + \frac{\gamma}{1 - \gamma} \right)\left(\gamma^k \|V_1 - V_0\| + \|\epsilon_k\| + \sum_{t = 1}^{k - 1} \gamma^t \|\epsilon_{k - t + 1} - \epsilon_{k - t}\| \right) + \gamma^k \|V_0 - V^*\| + \sum_{t = 0}^{k - 1} \gamma^t \|\epsilon_{k - t}\|
\end{align*}
\end{proof}
In \Cref{thm:approx-vipe}, we can have $V^{\pi_k} \to V^*$ if $\|\epsilon_k\| \to 0$, which takes care of the second term on the right-hand side. notice also that $\|\epsilon_k\| \to 0$ implies that $\|\epsilon_k\|$, and thus $\|\epsilon_{k - t + 1} - \epsilon_{k - t}\|$, is bounded for all $k$, which takes care of the third term. Of course, the first term goes to $0$ as $k \to \infty$. 

Notice also that in \Cref{thm:vipe} and \Cref{thm:approx-vipe}, we did not actually use that $\pi_k \in \argmin_\pi \|\bellman^\pi V_k - V_k\|$; we only used that for all $k$, $\|\bellman^{\pi_k}V_k - V_k\| \leq \|\bellman^{\pi^*} V_k - V_k\|$. This requirement is easier to satisfy than finding an evaluation policy, as one only has to improve upon the Bellman error with respect to the class of optimal policies, itself a much smaller class than the class of all policies. 

\section{Hyperparameter Sweeps for DQN}

\subsection{LunarLander-v2}
\begin{center}
    \begin{tabular}{ |c|c| } 
     \hline
     \textbf{Parameter} & \textbf{Value(s)} \\
     \hline
     Discount Rate ($\gamma$) & $0.99$ \\
     \hline
     Per Episode Frame Limit & $5000$ \\
     \hline
     Replay Buffer Size & $10^{5}$ \\
     \hline
     Mini-batch Size & $32$ \\
     \hline
     Target Net Update Period & $1, 100, \mathbf{500}$ \\
     \hline
     Initial $\epsilon$ & $1.0$ \\
     \hline
     Final $\epsilon$ & $0.2, 0.1, \mathbf{0.01}$ \\
     \hline
     $\epsilon$ Decay Steps & $\mathbf{25\textrm{k}}, 50\textrm{k}, 100\textrm{k}$ \\
     \hline
     Value Learning Rate & $10^{-5}, 10^{-4}, \mathbf{10^{-3}}$ \\
     \hline
     Hidden Layer Width & $64, 128, \mathbf{256}$ \\
     \hline
    \end{tabular}
\end{center}

\subsection{Freeway}
\begin{center}
    \begin{tabular}{ |c|c| } 
     \hline
     \textbf{Parameter} & \textbf{Value(s)} \\
     \hline
     Discount Rate ($\gamma$) & $0.99$ \\
     \hline
     Per Episode Frame Limit & $None$ \\
     \hline
     Replay Buffer Size & $10^{5}$ \\
     \hline
     Mini-batch Size & $32$ \\
     \hline
     Target Net Update Period & $1, \mathbf{100}$ \\
     \hline
     Initial $\epsilon$ & $1.0$ \\
     \hline
     Final $\epsilon$ & $0.2, 0.1, \mathbf{0.01}$ \\
     \hline
     $\epsilon$ Decay Steps & $\mathbf{100\textrm{k}}, 200\textrm{k}, 400\textrm{k}$ \\
     \hline
     Value Learning Rate & $\mathbf{10^{-5}}, 10^{-4}, 10^{-3}$ \\
     \hline
     Hidden Layer Width & $128$ \\
     \hline
    \end{tabular}
\end{center}

\section{Hyperparameter Sweeps for $\epsilon$-IPE}

\subsection{LunarLander-v2}
\begin{center}
    \begin{tabular}{ |c|c| } 
     \hline
     \textbf{Parameter} & \textbf{Value(s)} \\
     \hline
     Discount Rate ($\gamma$) & $0.99$ \\
     \hline
     Per Episode Frame Limit & $5000$ \\
     \hline
     Replay Buffer Size & $10^{5}$ \\
     \hline
     Mini-batch Size & $32$ \\
     \hline
     Target Net Update Period & $1, 100, \mathbf{500}$ \\
     \hline
     Policy Learning Rate & $10^{-5}, 10^{-4}, \mathbf{10^{-3}}$ \\
     \hline
     Value Learning Rate & $10^{-5}, 10^{-4}, \mathbf{10^{-3}}$ \\
     \hline
     Hidden Layer Width & $64, \mathbf{128}, 256$ \\
     \hline
    \end{tabular}
\end{center}

\subsection{Freeway}
\begin{center}
    \begin{tabular}{ |c|c| } 
     \hline
     \textbf{Parameter} & \textbf{Value(s)} \\
     \hline
     Discount Rate ($\gamma$) & $0.99$ \\
     \hline
     Per Episode Frame Limit & $None$ \\
     \hline
     Replay Buffer Size & $10^{5}$ \\
     \hline
     Mini-batch Size & $32$ \\
     \hline
     Target Net Update Period & $1, \mathbf{100}$ \\
     \hline
     Policy Learning Rate & $10^{-5}, 10^{-4}, \mathbf{10^{-3}}$ \\
     \hline
     Value Learning Rate & $\mathbf{10^{-5}}, 10^{-4}, 10^{-3}$ \\
     \hline
     Hidden Layer Width & $128$ \\
     \hline
    \end{tabular}
\end{center}

\section{Sensitivity Curves}

\subsection{LunarLander-v2}

\begin{figure}[!htb]
    \centering
    \begin{subfigure}[b]{0.3\linewidth}
        \centering
        \includegraphics[width=\textwidth]{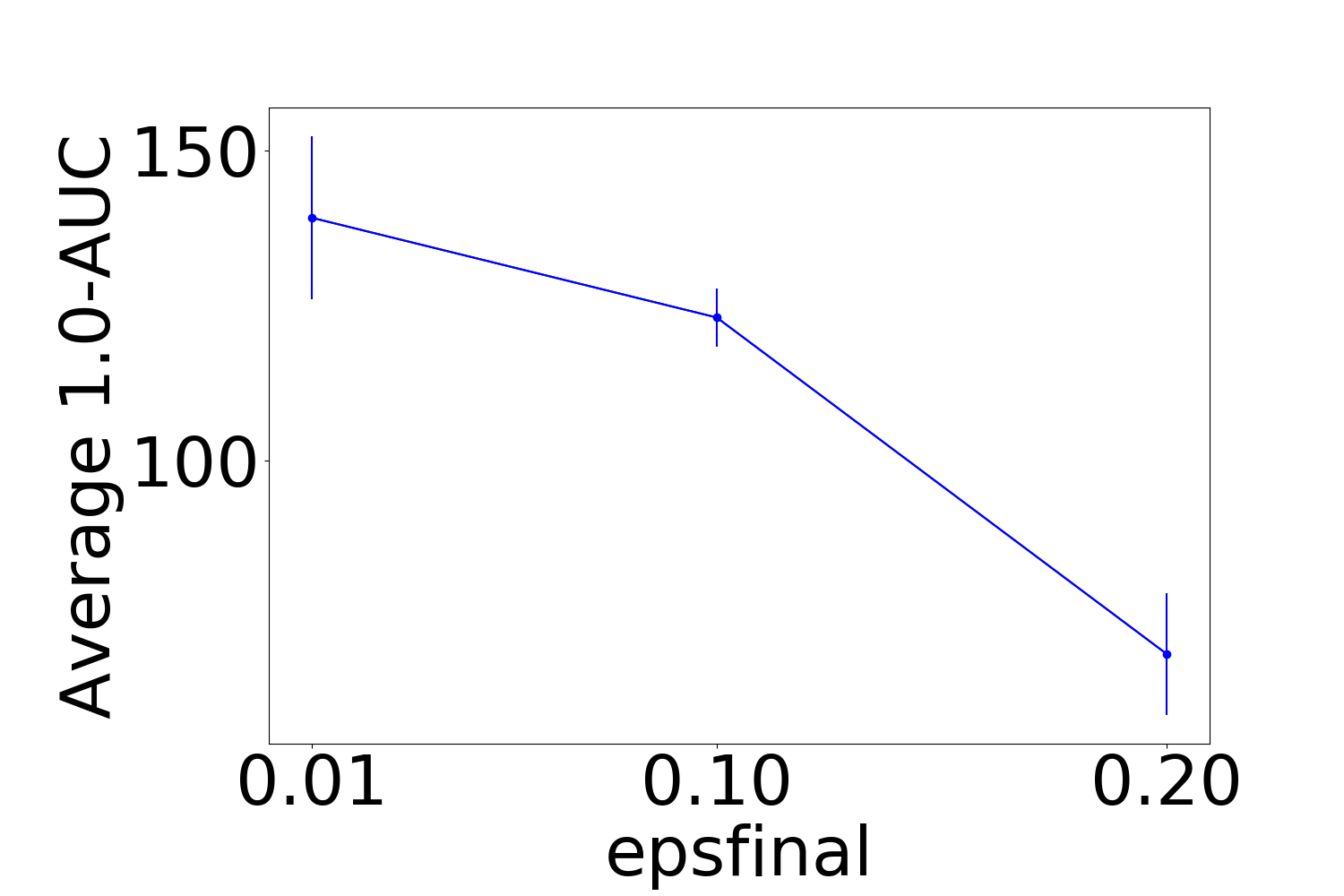}
        \caption{Final $\epsilon$}
    \end{subfigure}
    \begin{subfigure}[b]{0.3\linewidth}
        \centering
        \includegraphics[width=\textwidth]{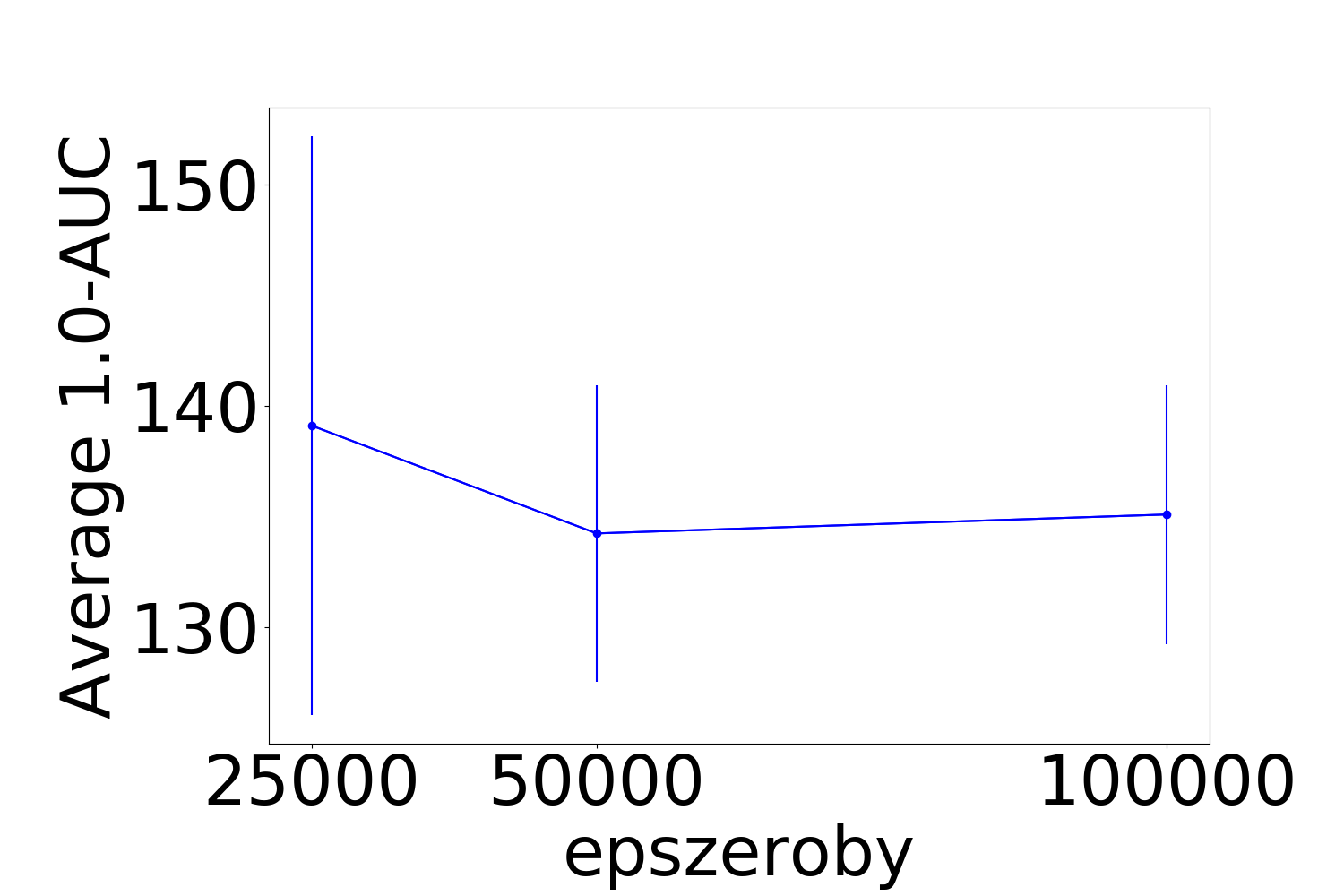}
        \caption{$\epsilon$ Decay Steps}
    \end{subfigure}
    \begin{subfigure}[b]{0.3\linewidth}
        \centering
        \includegraphics[width=\textwidth]{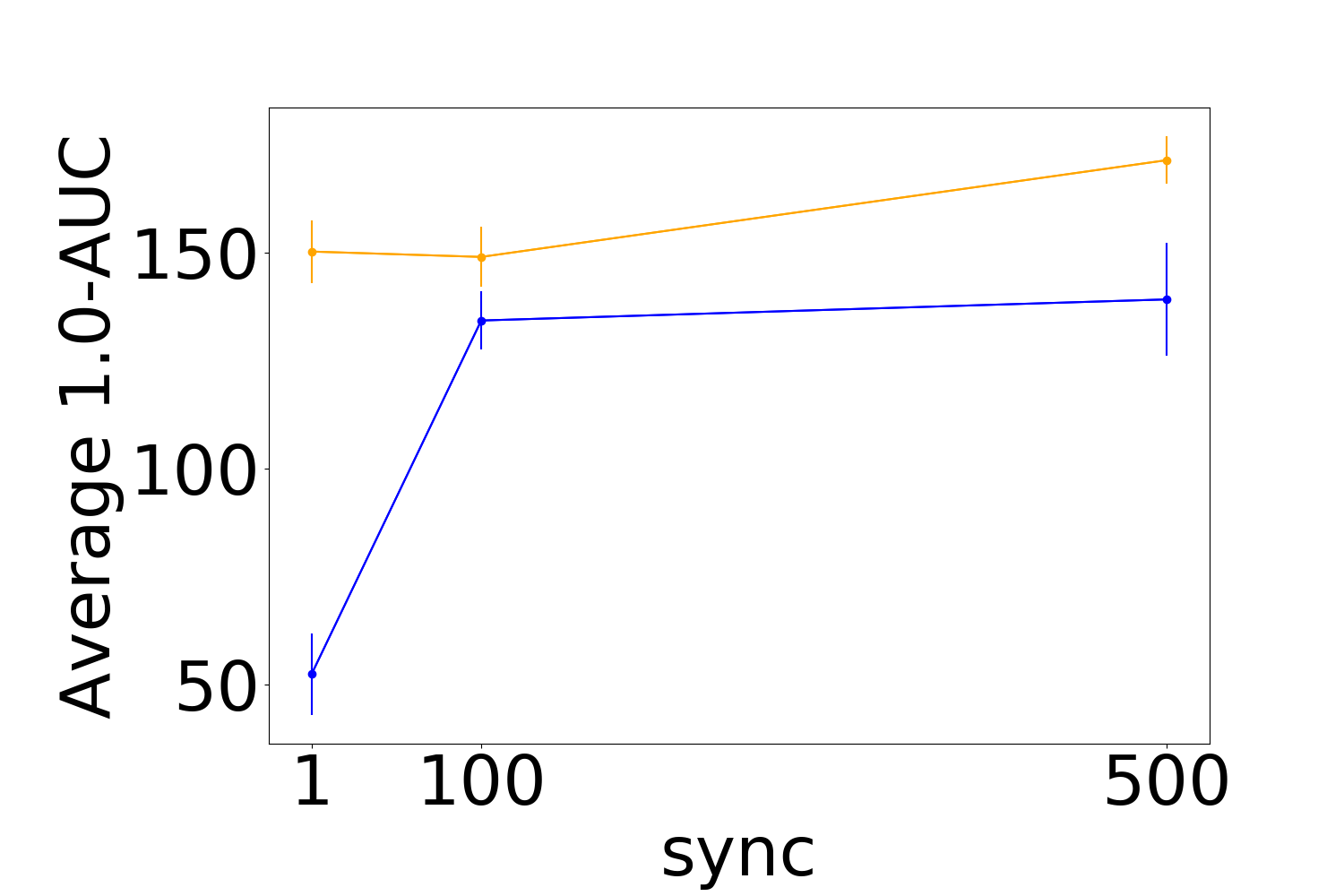}
        \caption{Target Net Update Period}
    \end{subfigure}
    \begin{subfigure}[b]{0.3\linewidth}
        \centering
        \includegraphics[width=\textwidth]{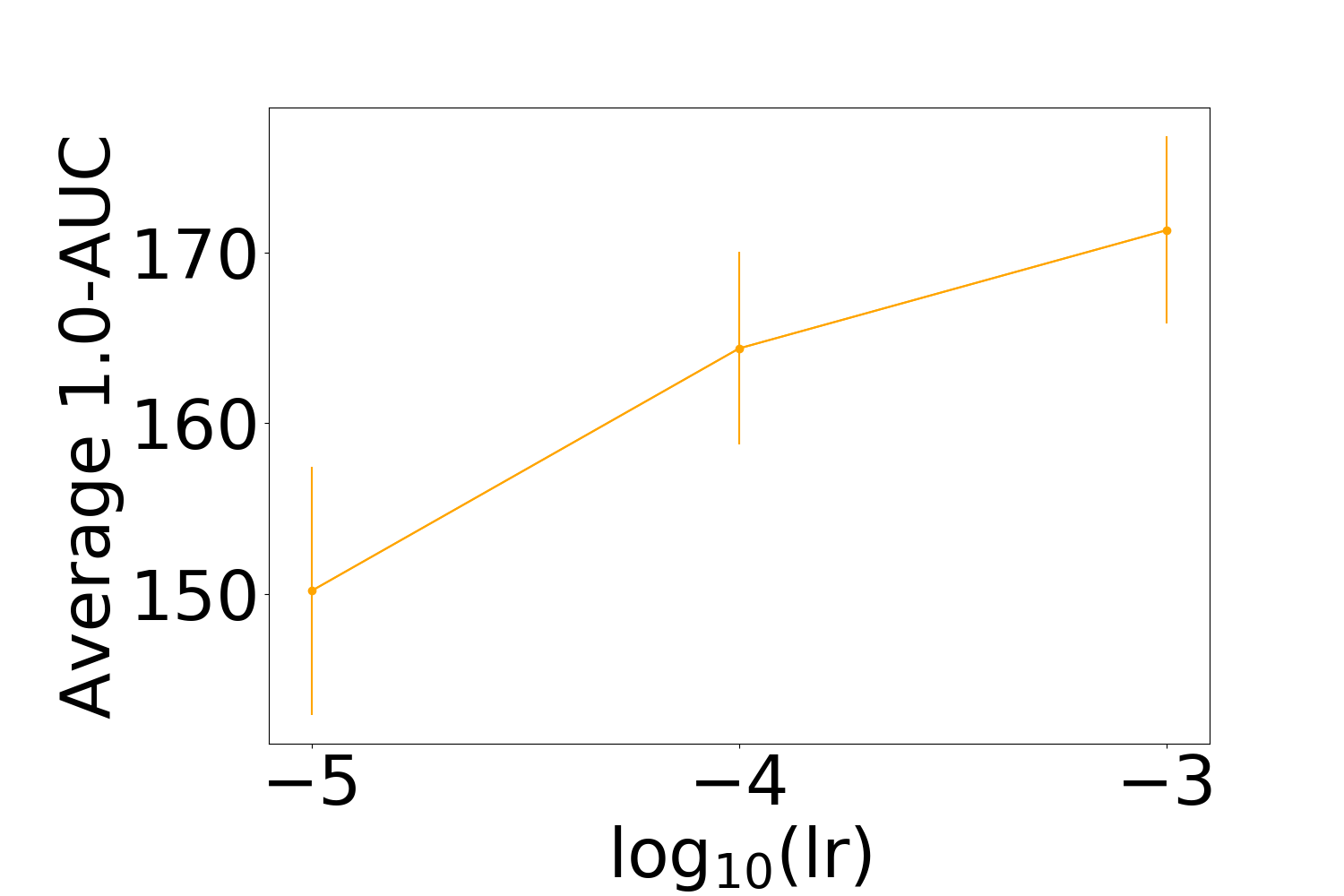}
        \caption{Policy Learning Rate}
    \end{subfigure}
    \begin{subfigure}[b]{0.3\linewidth}
        \centering
        \includegraphics[width=\textwidth]{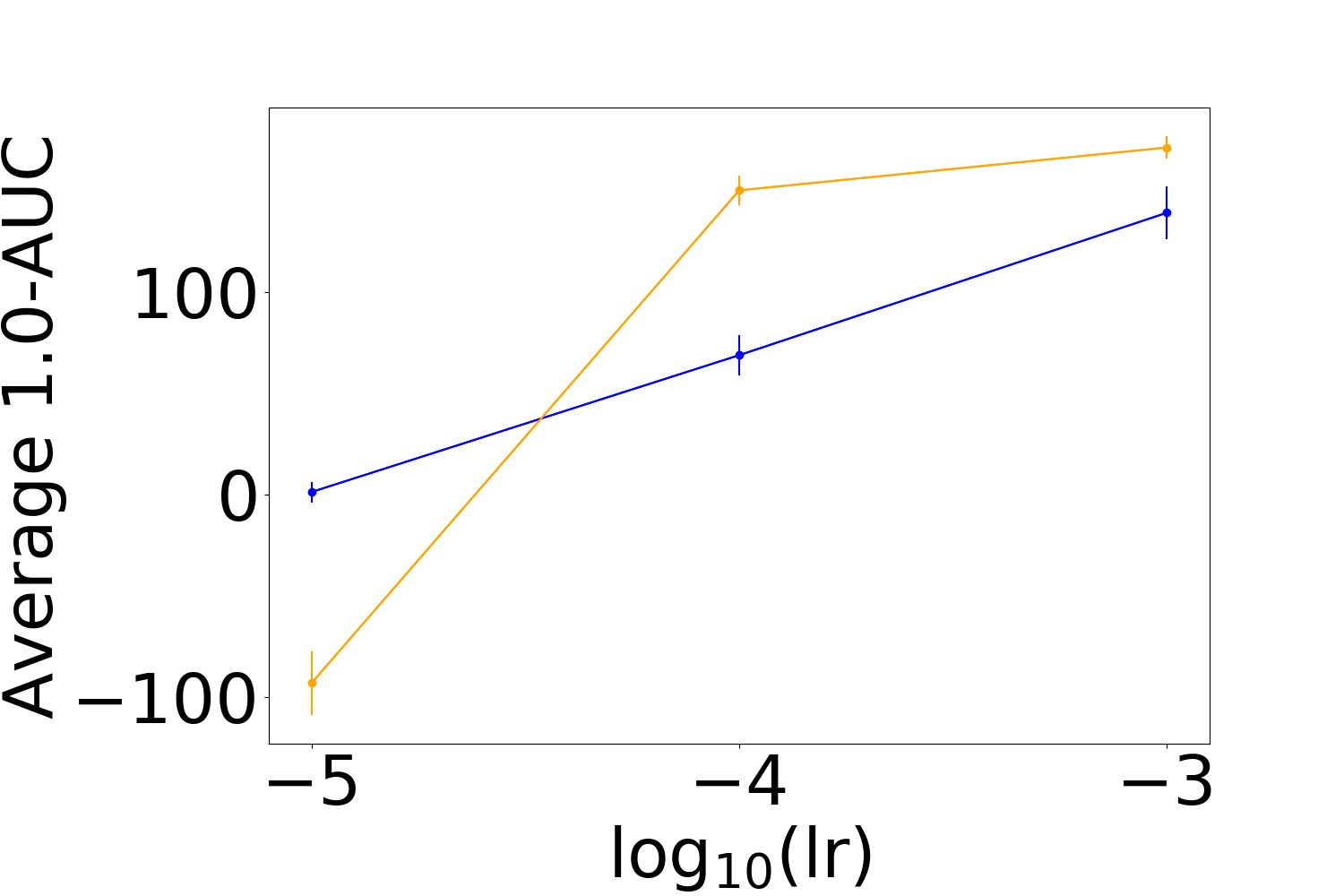}
        \caption{Value Learning Rate}
    \end{subfigure}
    \begin{subfigure}[b]{0.3\linewidth}
        \centering
        \includegraphics[width=\textwidth]{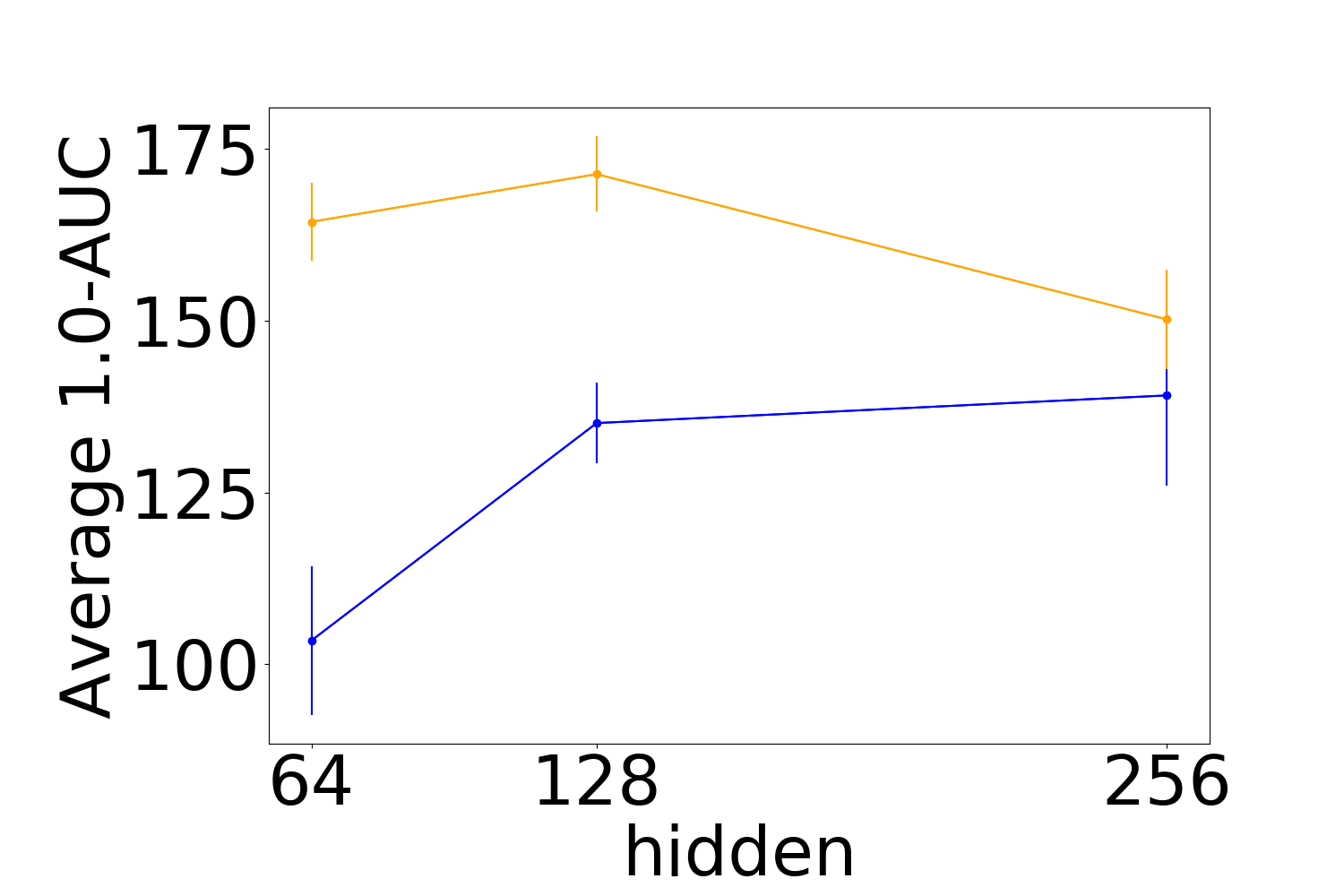}
        \caption{Hidden Layer Width}
    \end{subfigure}
    \caption{Sensitivity curves of DQN with $\epsilon$-greedy (Blue) and DQN with $\epsilon$-IPE (Orange). Each point is the area under the curve for an individual parameter setting, using the best setting of the remaining parameters under the same metric, averaged over 30 runs. Error bars represent standard error.}
\end{figure}

\subsection{Freeway}

\begin{figure}[!htb]
    \centering
    \begin{subfigure}[b]{0.3\linewidth}
        \centering
        \includegraphics[width=\textwidth]{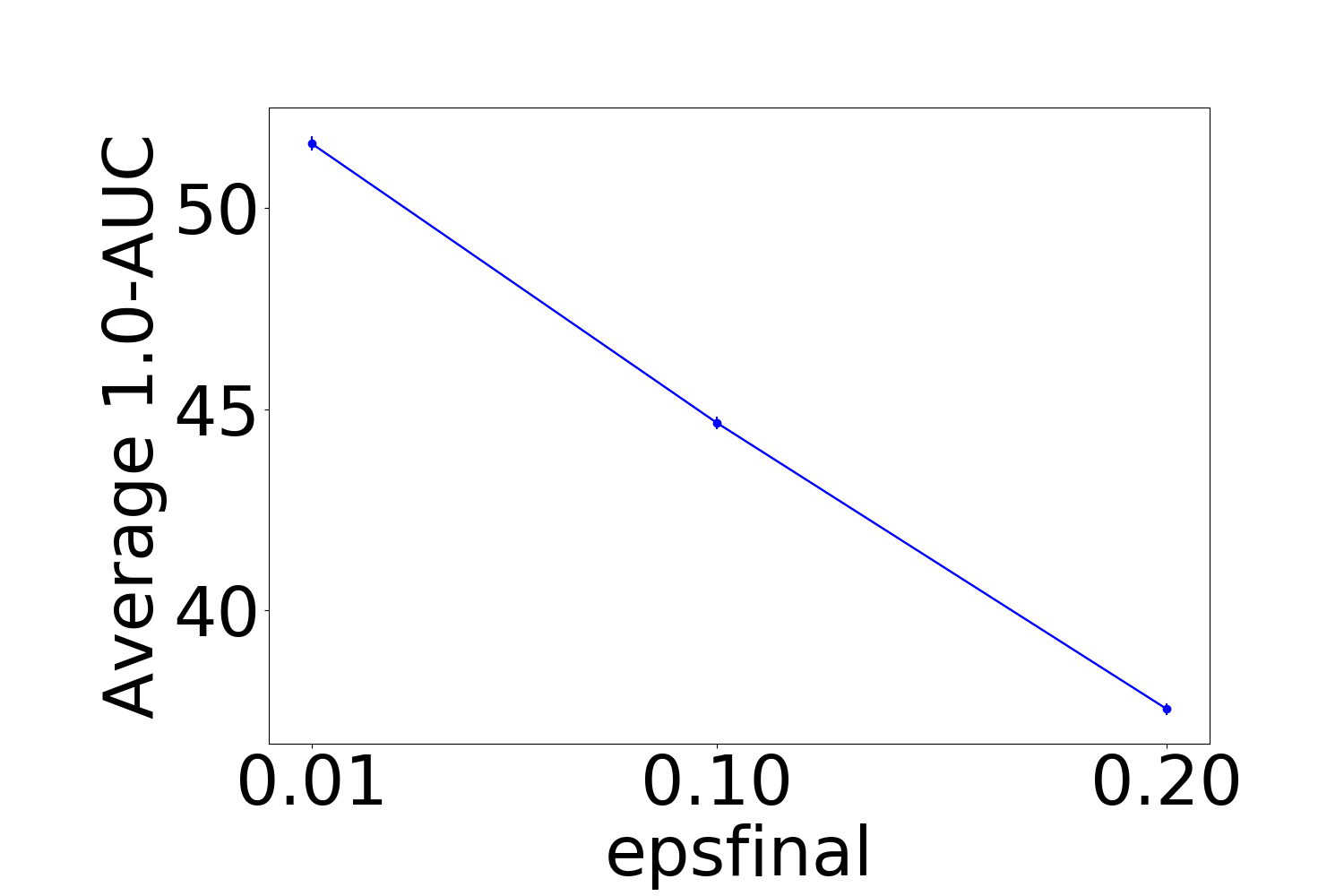}
        \caption{Final $\epsilon$}
    \end{subfigure}
    \begin{subfigure}[b]{0.3\linewidth}
        \centering
        \includegraphics[width=\textwidth]{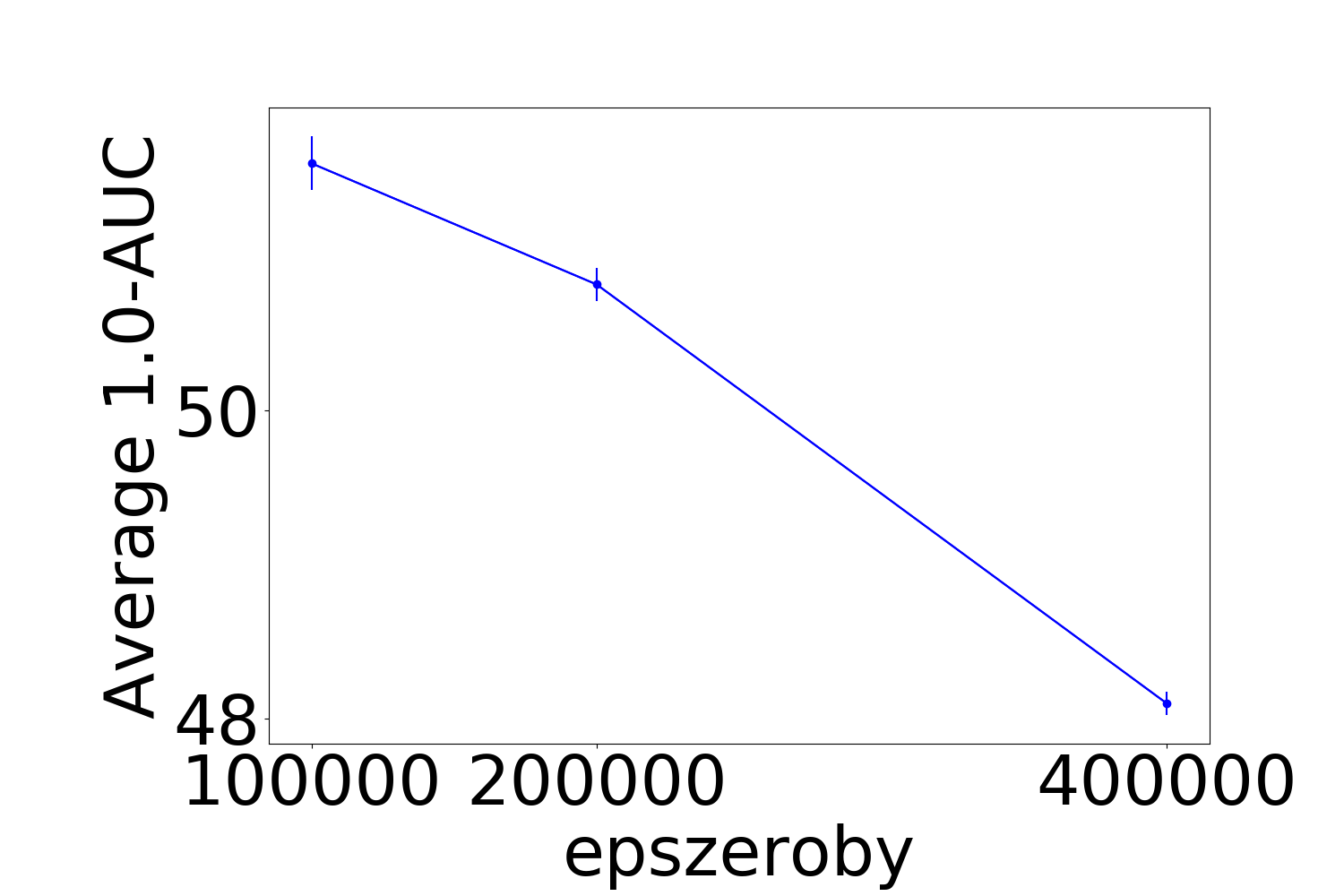}
        \caption{$\epsilon$ Decay Steps}
    \end{subfigure}
    \begin{subfigure}[b]{0.3\linewidth}
        \centering
        \includegraphics[width=\textwidth]{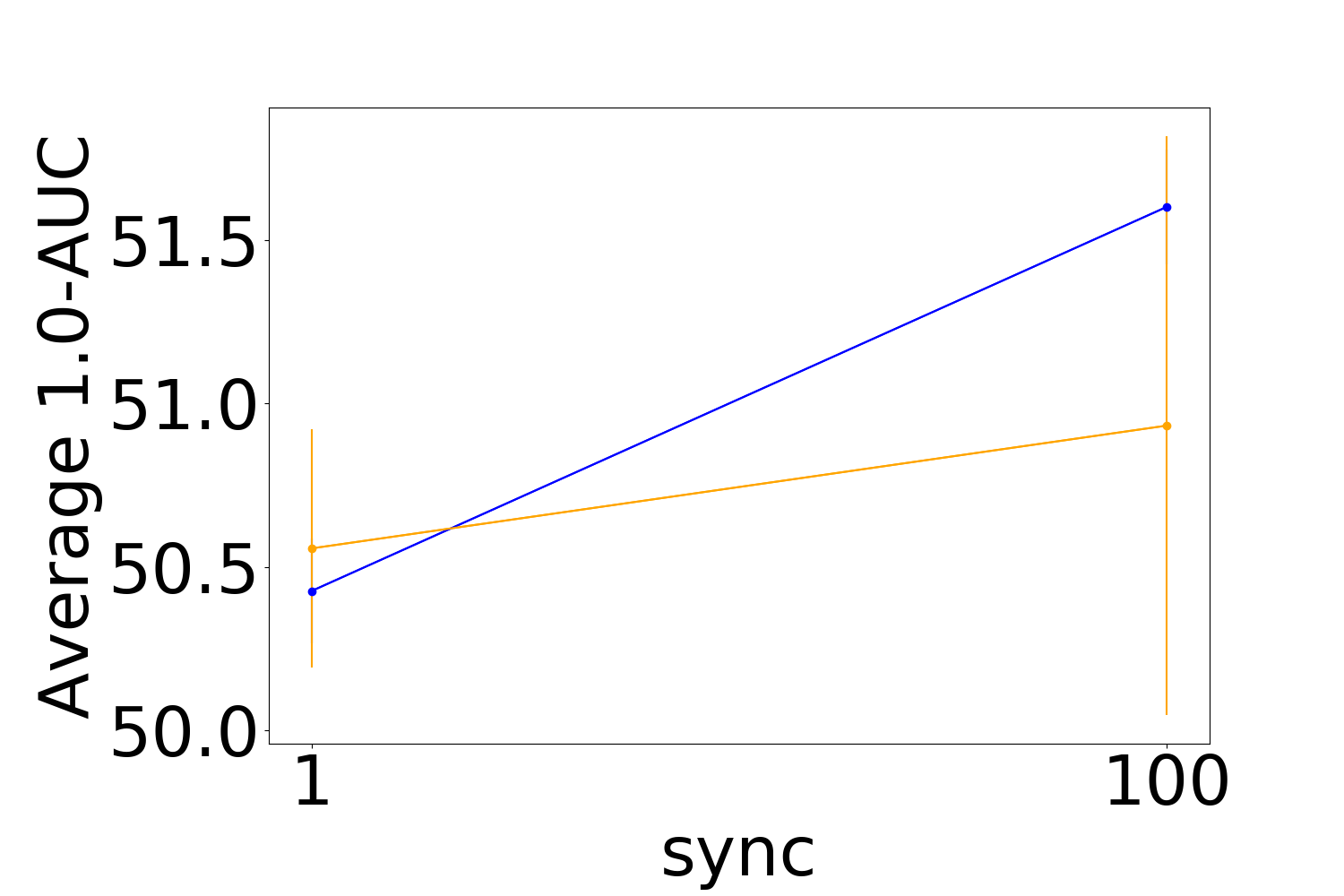}
        \caption{Target Net Update Period}
    \end{subfigure}
    \begin{subfigure}[b]{0.3\linewidth}
        \centering
        \includegraphics[width=\textwidth]{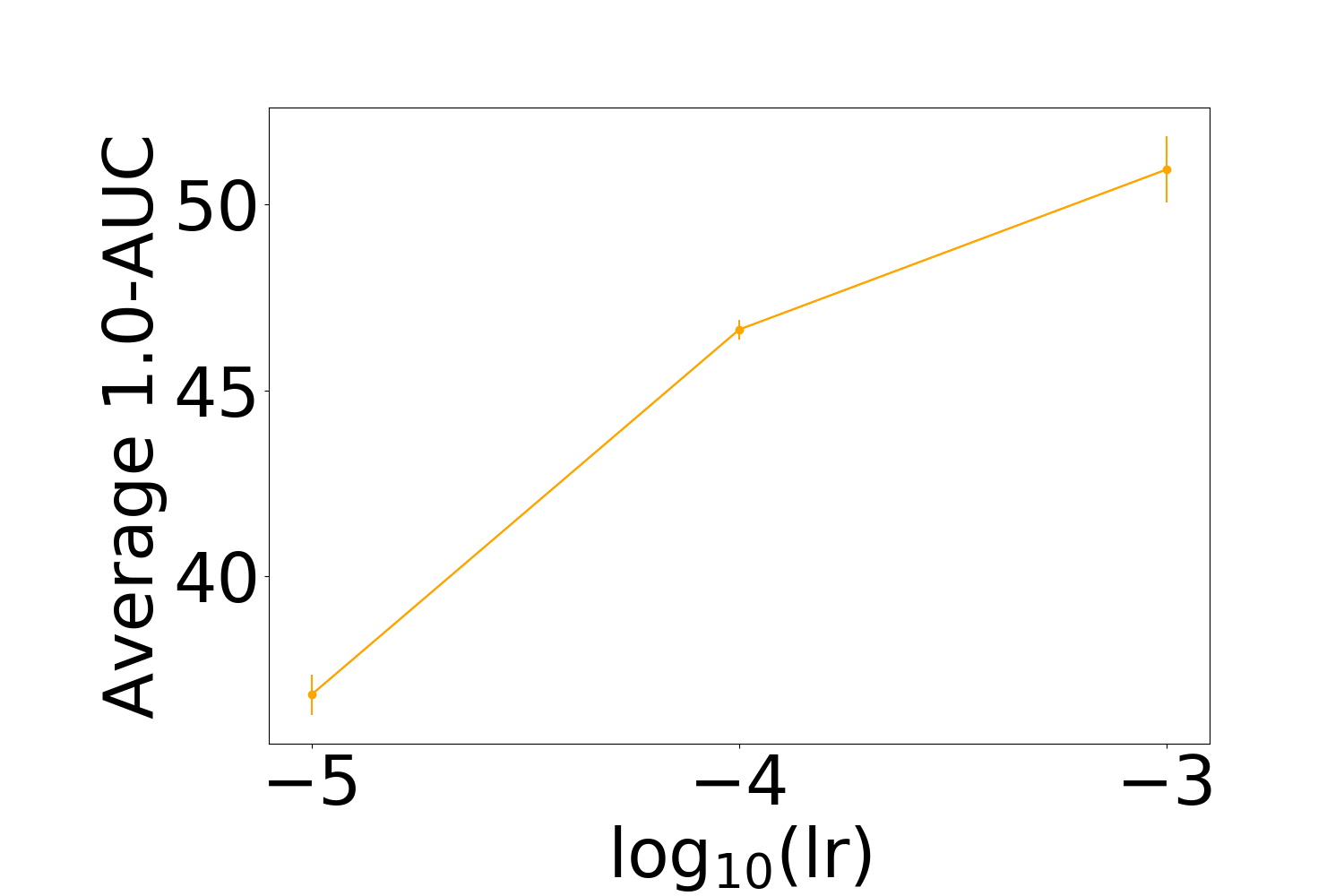}
        \caption{Policy Learning Rate}
    \end{subfigure}
    \begin{subfigure}[b]{0.3\linewidth}
        \centering
        \includegraphics[width=\textwidth]{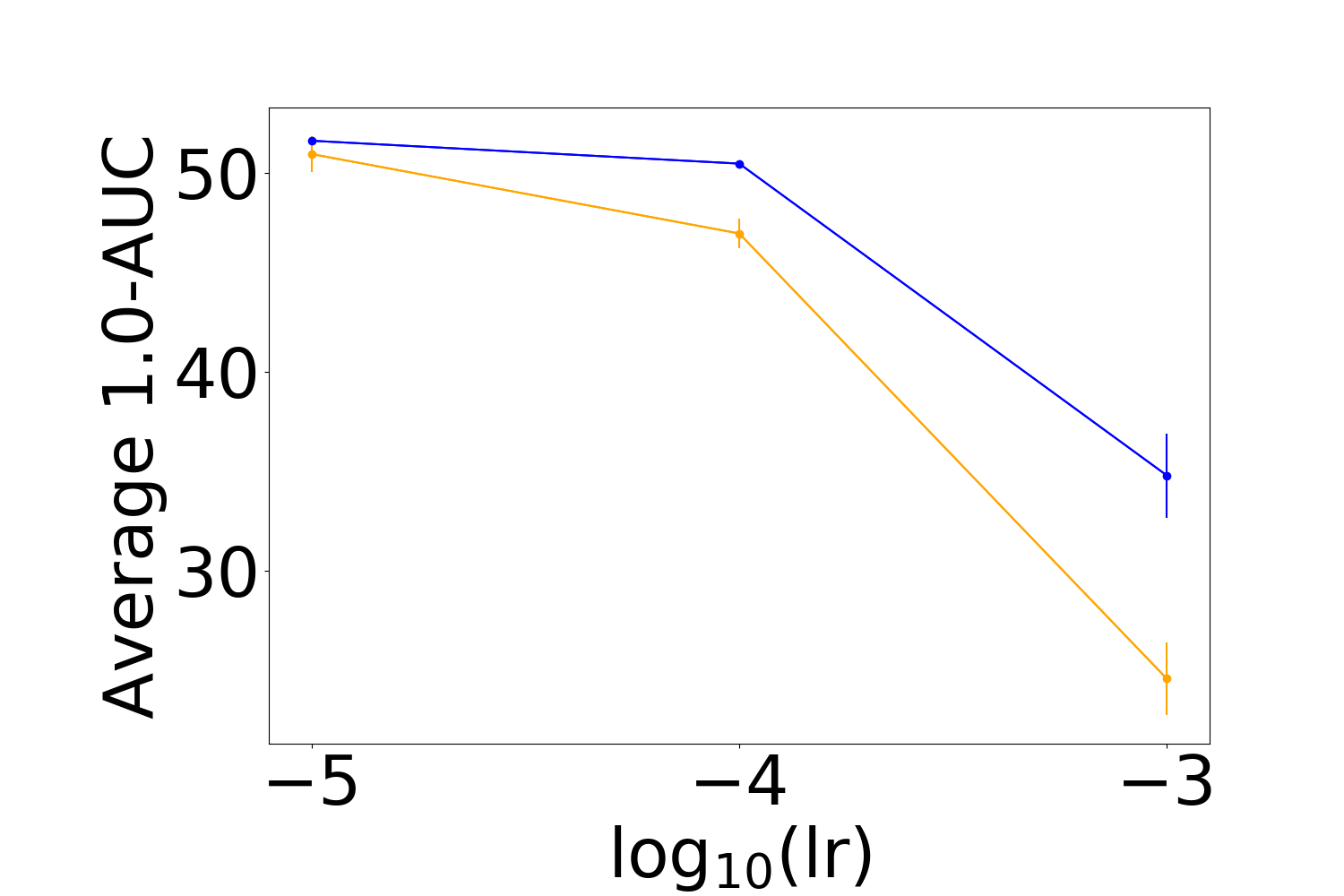}
        \caption{Value Learning Rate}
    \end{subfigure}
    \caption{Sensitivity curves of DQN with $\epsilon$-greedy (Blue) and DQN with $\epsilon$-IPE (Orange). Each point is the area under the curve for an individual parameter setting, using the best setting of the remaining parameters under the same metric, averaged over 30 runs. Error bars represent standard error.}
\end{figure}

\end{document}